\documentclass{article}

\usepackage{graphicx}
\usepackage{subcaption}
\usepackage[margin=1.25in]{geometry}
\usepackage{graphicx}
\usepackage{color,soul}

\usepackage{float}
\usepackage{subfloat}

\usepackage[utf8]{inputenc} 
\usepackage[T1]{fontenc}    
\usepackage{url}            
\usepackage{booktabs}       
\usepackage{amsfonts}       
\usepackage{nicefrac}       
\usepackage{microtype}      
\usepackage{tcolorbox}
\usepackage[hidelinks]{hyperref}
\definecolor{darkred}{RGB}{150,0,0}
\definecolor{darkgreen}{RGB}{0,150,0}
\definecolor{darkblue}{RGB}{0,0,150}
\hypersetup{colorlinks=true, linkcolor=red, citecolor=darkgreen, urlcolor=darkblue}

\usepackage{amssymb}
\usepackage{amsmath}
\usepackage{amsmath,amsfonts,amssymb,amsthm}
\usepackage{mathtools}
\usepackage{commath}
\usepackage{flexisym}
\usepackage{csquotes}
\usepackage[english]{babel}
\usepackage[utf8]{inputenc}
\usepackage{algorithm}
\usepackage[noend]{algpseudocode}
\usepackage{color}

\usepackage{amsmath}

\DeclareMathOperator*{\argmin}{arg\,min}

\newtheorem{lemma}{Lemma}
\newtheorem{theorem}{Theorem}
\newtheorem{myassum}{Assumption}
\newtheorem{remark}{Remark}

\newcommand*{\affaddr}[1]{#1} 
\newcommand*{\affmark}[1][*]{\textsuperscript{#1}}
\newcommand*{\email}[1]{\texttt{#1}}

\title{Linear Stochastic Bandits Under Safety Constraints}

\newcommand{\Dts}{\Dc_t^{\text{s}}}
\newcommand{\Ds}{\Dc_0^{\text{s}}}
\newcommand{\Dtss}{{\widetilde{\Dc}}_t^{\text{s}}}
\newcommand{\Ats}{\Ac_t^{\text{s}}}
\newcommand{\SUCB}{Safe-LUCB}
\newcommand{\Dw}{\Dc^w}
\newcommand{\simiid}{\stackrel{\rm iid}{\sim}}

\newcommand{\lamin}{\la_{\rm \min}}
\newcommand{\lamax}{\la_{\rm \max}}
\newcommand{\la}{\lambda}

\newcommand{\R}{\mathbb{R}}

\newcommand{\nn}{\nonumber}

\newcommand{\Otilde}{\widetilde{\mathcal{O}}}

\newcommand{\eps}{\epsilon}

\newcommand{\bal}{\begin{align}}
\newcommand{\eal}{\end{align}}

\DeclarePairedDelimiterX{\inp}[2]{\langle}{\rangle}{#1, #2}

\newcommand{\Sc}{{\mathcal{S}}}
\newcommand{\Bc}{{\mathcal{B}}}

\newcommand{\Dc}{\mathcal{D}}

\newcommand{\Cc}{\mathcal{C}}

\newcommand{\Ac}{\mathcal{A}}

\newcommand{\Ec}{\mathcal{E}}

\newcommand{\Oc}{\mathcal{O}}

\newcommand{\beq}{\begin{equation}}
\newcommand{\eeq}{\end{equation}}
\newcommand{\bea}{\begin{align}}
\newcommand{\eea}{\end{align}}

\newcommand{\E}{\mathbb{E}}

\author{%
Sanae Amani\affmark[1], Mahnoosh Alizadeh\affmark[1], and  Christos Thrampoulidis\affmark[1]\\
\affaddr{\affmark[1] University of California, Santa Barbara}\\
\email{\{samanigeshnigani,alizadeh,cthrampo\}@ucsb.edu}\\
}

\begin{document}

\date{}
\maketitle
\begin{abstract}
Bandit algorithms have various application in safety-critical systems, where it is important to respect the system constraints that rely on the bandit's unknown parameters at every round. In this paper, we formulate a linear stochastic multi-armed bandit problem with safety constraints that depend (linearly) on an unknown parameter vector. As such, the learner is unable to identify all safe actions and must act conservatively in ensuring that her actions satisfy the safety constraint at all rounds (at least with high probability). For these bandits, we propose a new UCB-based algorithm called Safe-LUCB, which includes necessary modifications to respect safety constraints. The algorithm has two phases. During the pure exploration phase the learner chooses her actions at random from a restricted set of safe actions with the goal of learning a good approximation of the entire unknown safe set. Once this goal is achieved, the algorithm begins a safe exploration-exploitation phase where the learner gradually expands their estimate of the set of safe actions while controlling the growth of regret. We provide a general regret bound for the algorithm, as well as a problem dependent bound that is connected to the location of the optimal action within the safe set. We then propose a modified heuristic that exploits our problem dependent analysis to improve the regret.

\end{abstract}

\section{Introduction}
The stochastic multi-armed bandit (MAB) problem is a sequential decision-making problem where, at each step  of a $T$-period run, a learner plays one of $k$ arms and observes a corresponding loss that is sampled independently from  an underlying distribution with unknown parameters. The learner's goal is to minimize the pseudo-regret, i.e., the
difference between the expected $T$-period loss incurred by the  decision making algorithm
and the optimal loss if the unknown parameters were given.
The linear stochastic bandit problem generalizes MAB to the setting where each arm is associated with a feature vector $x$  and the expected loss of each arm is equal to the inner product of its feature vector $x$ and an unknown parameter vector $\mu$.  There are several variants of linear stochastic bandits that consider finite or infinite number of arms, as well as the case where the set of feature vectors changes over time. A detailed account of previous work in this area will be provided in Section \ref{lit.review}.

Bandit algorithms have found many applications in systems that repeatedly deal with unknown stochastic environments (such as humans) and seek to optimize a long-term reward by simultaneously learning and exploiting the unknown environment (e.g., ad display optimization algorithms with unknown user preferences, path routing, ranking in search engines). They are also naturally relevant for many cyber-physical systems with humans in the loop (e.g., pricing end-use demand in societal-scale infrastructure systems such as power grids or transportation networks to minimize system costs given the limited number of user interactions possible). However, existing bandit heuristics might not be directly applicable in these latter cases.  One critical reason is the existence of safety guarantees that have to be met at every single round. For example, when managing demand to minimize costs in a power system, it is required that the operational constraints of the power grid are not violated in response to our actions (these can be formulated as linear constraints that depend on the demand). Thus, for such systems, it becomes important to develop new bandit algorithms that account for  critical safety requirements.

%
Given the high level of uncertainty about the system parameters in the initial rounds, any such bandit algorithm will be initially highly constrained in terms of safe actions that can be chosen. However, as further samples are obtained and the algorithm becomes more confident about the value of the unknown parameters, it is intuitive that safe actions become easier to distinguish and it seems plausible that the effect of the system safety requirements on the growth of regret can be diminished.

In this paper, we formulate a variant of linear stochastic bandits where at each round $t$, the learner's choice of arm should also satisfy a  {\it safety constraint} that is dependent on the unknown parameter vector $\mu$. While the formulation presented is certainly an abstraction of the complications that might arise in the systems discussed above, we believe that it is a natural first step towards understanding and evaluating the effect of safety constraints on the performance of bandit heuristics.

Specifically, we assume that the learner's goal is twofold: 1) Minimize the $T$-period cumulative pseudo-regret; 2) Ensure that a linear side constraint of the form $\mu^\dagger B x \leq c$ is respected at every round during the $T$-period run of the algorithm, where $B$ and $c$ are known. See Section \ref{sec:form} for details. Given the learner's uncertainty about $\mu$, the existence of this safety constraint effectively restricts the learner's choice of actions to what we will refer to as the {\it safe decision set} at each round $t$. To tackle this constraint, in Section \ref{sec:algo}, we present  \SUCB~as a safe version of the standard linear UCB (LUCB) algorithm \cite{Dani08stochasticlinear,abbasi2011improved,Tsitsiklis}. In Section \ref{sec:theory} we provide general regret bounds that characterize the effect of safety constraints on regret. We show that  the regret of the modified algorithm is dependent on the parameter $\Delta = c -  \mu^\dagger B x^*,$ where $x^*$ denotes the optimal safe action given $\mu$. When $\Delta>0$ and is known to the learner, we show that the regret of \SUCB~is $\Otilde(\sqrt{T})$; thus, the effect of the system safety requirements on the growth of regret can be diminished (for large enough $T$). In Section \ref{sec:GSUCB}, we also present a heuristic modification of \SUCB~that empirically approaches the same regret without a-priori knowledge of the value of $\Delta$. On the other hand, when $\Delta=0$, the regret of \SUCB~is $\Otilde(T^{2/3})$. Technical proofs and some further discussions are deferred to the appendix provided in the supplementary material.

%

\noindent{\bf Notation.} The Euclidean norm of a vector $x$ is denoted by $\| x\|_2$ and the spectral norm of a matrix $M$ is denoted by $\|M\|$. We denote the transpose of any column vector $x$ by $x^{\dagger}$. Let $A$ be a positive definite $d\times d$ matrix and $v\in \mathbb R^d$. The weighted 2-norm of $v$ with respect to $A$ is defined by $\|v\|_A = \sqrt{v^\dagger A v}$. We denote the minimum and maximum eigenvalue of $A$ by $\lamin(A)$ and $\lamax(A)$. The maximum of two numbers $\alpha, \beta$ is denoted $\alpha\vee\beta$. For a positive integer $n$, $[n]$ denotes the set $\{1,2,\ldots,n\}$. 
Finally, we use standard $\Otilde$ notation for big-O notation that ignores logarithmic factors.

\subsection{Safe linear stochastic bandit problem}\label{sec:form}

\vspace{-0.1cm}
\noindent{\bf Cost model.}
 The learner is given a  convex compact decision set $\Dc_0 \subset \mathbb R^d$. 
 At each round $t$,  the learner chooses an action $x_t \in \Dc_0$ which results in an observed loss $\ell_t$ that is linear on the unknown parameter $\mu$ with additive random noise $\eta_t$, i.e.,
$\ell_t := c_t(x_t) := \mu^{\dagger}x_t + \eta_t.$

\noindent{\bf Safety Constraint.} The learning environment is subject to a side constraint that restricts the choice of actions by dividing $\Dc_0$ into a safe and an unsafe set. The learner is restricted to actions $x_t$ from the \emph{safe set} $\Ds(\mu)$.  As  notation suggests, the safe set depends on the unknown parameter. Since $\mu$ is unknown, the learner is unable to identify the safe set and must act conservatively in ensuring that actions $x_t$ are feasible for all $t$. 
In this paper, we assume that $\Ds(\mu)$ is defined via a linear constraint
\begin{equation} \label{eq:constraint}
      \mu^\dagger B x_t \leq c,
\end{equation}
which needs to be satisfied by $x_t$ at all rounds $t$ with high probability. Thus, $\Ds(\mu)$ is defined as,
\begin{align}\label{eq:Ds}
    \Ds(\mu):=\{x\in \Dc_0\,:\, \mu^\dagger B x \leq c\}.
\end{align}
The matrix $B \in \mathbb{R}^{d \times d}$ and the positive constant $c>0$ are known to the learner. However, after playing any action $x_t$, the value $\mu^\dagger B x_t$ is \emph{not} observed by the learner. When clear from context, we  drop the argument $\mu$ in the definition of the safe set and simply refer to it as $\Ds$.


%

\noindent{\bf Regret.} Let $T$ be the total number of rounds. 
If $x_t,~t\in[T]$ are the actions chosen, then the \textit{cumulative pseudo-regret} (\cite{audibert2009exploration}) of the learner's algorithm for choosing the actions $x_t$ is defined by
    $R_T = \sum_{t=1} ^T \mu^\dagger x_t - \mu^\dagger x^*,$  
where $x^*$ is the optimal \emph{safe} action that minimizes the loss $\ell_t$ in expectation, i.e.,
$ x^* \in \argmin_{x\in \Ds(\mu)}\mu^\dagger x.$

\noindent{\bf Goal.} The goal of the learner is to keep $R_T$ as small as possible. At the bare minimum, we
require that the algorithm leads to $R_T/T\rightarrow 0$ (as $T$ grows large). In contrast to existing linear stochastic bandit formulations, we require that the chosen actions $x_t, t\in[T]$ are safe (i.e., belong in $\Ds$ \eqref{eq:Ds}) with high probability. For the rest of this paper, we simply use regret to refer to the pseudo-regret $R_T$. 

In Section  \ref{sec:ass} we place some further technical assumptions on $\Dc_0$ (bounded), on $\Ds$ (non-empty), on $\mu$ (bounded) and on the distribution of $\eta_t$ (subgaussian). 

\vspace{-0cm}
\subsection{Related Works}\label{lit.review}
\vspace{-0cm}
Our algorithm relies on a modified version of the famous UCB alogithm known as UCB1, which was first developed by \cite{Auer}. For linear stochastic bandits, the regret of 
the LUCB algorithm was analyzed by, e.g., \cite{Dani08stochasticlinear,abbasi2011improved,Tsitsiklis,russo,pmlr-v15-chu11a} and it was shown that the regret grows at the rate of $\sqrt{T} \log(T)$. Extensions to generalized linear bandit models have also been considered by, e.g., \cite{filippi2010parametric,li2017provably}.
There are two different contexts where constraints have been applied to the stochastic MAB problem. 
The first line of work considers the MAB problem with  global budget (a.k.a. knapsack) constraints where each arm is associated with a random resource consumption and the objective is to maximize the total reward  before the learner runs out of resources, see, e.g.,  \cite{knap1,knap2,knap3,knap4}. 
The second line of work   considers stage-wise safety  for bandit problems in the context of ensuring that the algorithm's regret performance stays above a fixed percentage of the performance
of a baseline strategy at every round during its run \cite{vanroy,wu}. In \cite{vanroy}, 
which is most closely related to our setting,
 the authors study a variant of LUCB in which the chosen actions are constrained such that the \emph{cumulative} reward remains \emph{strictly} greater than $(1-\alpha)$ times a given baseline reward for all  $t$. In both of the above mentioned lines of work, the   constraint   applies to the cumulative resource consumption (or reward) across the entire run of the algorithm. As such, the set of permitted actions at each round
 vary depending on the round and on the history of the algorithm.
 This is unlike our constraint, which is applied at each individual round, is deterministic, and does \emph{not} depend on the history of past actions.

In a more general context, the concept of safe learning has received significant attention in recent years from different communities. 
Most existing work that consider mechanisms for {\it safe exploration}  in unknown and stochastic environments are in  reinforcement learning or control. However, the notion of safety has many diverse definitions in this literature.  For example,  \cite{moldovan2012safe} proposes an algorithm that allows safe exploration in Markov Decision Processes (MDP) in order to avoid fatal absorbing states  that must never be visited during the exploration process. By considering constrained MDPs that are augmented with a set of auxiliary cost functions and replacing them with surrogates that are easy to estimate,  \cite{achiam2017constrained} purposes a policy search algorithm for constrained
reinforcement learning with guarantees for near constraint satisfaction at each iteration.
In the framework of global optimization or active data selection, \cite{GP1,berkenkamp2016bayesian} assume that the underlying system is safety-critical and present   active learning frameworks that use Gaussian Processes (GP) as non-parametric models to learn the safe decision set. More closely related to our setting,
\cite{Krause,stagewise}  extend the application of UCB to \emph{nonlinear} bandits with nonlinear constraints modeled through Gaussian processes (GPs). The algorithms in \cite{Krause,stagewise}  come with convergence guarantees, but \emph{no} regret bounds as provided in our paper. Regret guarantees imply convergence guarantees from an optimization perspective (see \cite{srinivasgaussian}), \emph{but not the other way around}. Such approaches for safety-constrained optimization using GPs have shown great promise in robotics applications with safety constraints \cite{ostafew2016robust,7039601}.
With a control theoretic point of view, \cite{tomlin2} combines reachability analysis and machine learning for autonomously
learning the dynamics of a target vehicle and \cite{tomlin} designs a learning-based MPC scheme that provides
deterministic guarantees on robustness when the underlying system model
is linear and has a known level of uncertainty. In a very recent related work \cite{kamgar}, the authors propose and analyze a (safe) variant of the Frank-Wolfe algorithm to solve a smooth optimization problem with unknown linear constraints that are accessed by the learner via stochastic zeroth-order feedback. The main goal in \cite{kamgar} is to provide a convergence rate for more general convex objective, whereas we aim to provide \emph{regret bounds} for a linear but otherwise unknown objective.




%
\vspace{-0.1cm}
\section{A  \SUCB~Algorithm}\label{sec:algo}
\vspace{-0.2cm}
Our proposed algorithm is a safe version of LUCB. As such, it relies on the well-known heuristic principle of \textit{optimism in the face of uncertainty} (OFU).
%
The algorithm constructs a confidence set $\Cc_t$ at each round $t$, within which  the unknown parameter $\mu$ lies with high probability. In the absence of any constraints, the  learner chooses the most ``favorable'' environment $\mu$ from the set $\Cc_t$ and plays the action $x_t$ that minimizes the expected loss in that environment. \
However, the presence of the constraint \eqref{eq:constraint}
complicates the choice of the learner. To address this, we propose an algorithm called \emph{safe linear upper confidence bound} (\SUCB), which attempts to minimize regret while making sure that the safety constraints \eqref{eq:constraint} are satisfied.
\SUCB~is summarized in Algorithm \ref{algo:Safe-LUCB} and a detailed presentation follows in Sections   \ref{sec:Phase I} and \ref{sec:PhaseII}, where we discuss the \emph{pure-exploration} and \emph{safe exploration-exploitation} phases  of the algorithm, respectively.  Before these, in Section \ref{sec:ass} we introduce the necessary conditions under which our proposed algorithm operates and achieves good regret bounds as will be shown in Section \ref{sec:theory}.

\subsection{Model Assumptions}\label{sec:ass}
Let $\mathcal{F}_t = \sigma(x_1,x_2,\ldots,x_{t+1},\eta_1,\eta_2,\ldots,\eta_{t})$ be the $\sigma$-algebra (or, history) at round $t$. 
We make the following standard assumptions on the noise distribution, on the parameter $\mu$ and on the actions. 
 
 \begin{myassum}[Subgaussian Noise] \label{assum:noise}
For all $t$, $\eta_t$ is conditionally zero-mean $R$-sub-Gaussian for fixed constant $R\geq 0$, i.e., $\E[\eta_t\,|\,x_{1:t},\eta_{1:t-1}]=0$ and 
$\mathbb{E}[e^{\lambda \eta_t}\,|\,\mathcal{F}_{t-1}] \leq \exp(\lambda^2 R^2/2), \quad \forall \lambda \in \mathbb{R}.$
 \end{myassum}

\begin{myassum}[Boundedness]\label{assum:bounded}
There exist positive constants $S,L$ such that $\|\mu\|_2\leq S$ and $\|x\|_2 \leq L, \forall x \in \Dc_0$. Also,  $\mu^\dagger x \in [-1,1], \forall x \in \Dc_0$ . 
\end{myassum}
\vspace{-0.2cm}
In order to avoid trivialities, we also make the following assumption. This, together with the assumption that $c>0$ in \eqref{eq:constraint}, guarantee that the safe set  $\Ds(\mu)$ is non-empty (for every $\mu$). 
\begin{myassum}[Non-empty safe set] \label{assum:nonempty}
The decision set $\Dc_0$ is a convex body in $\R^d$ that contains the origin in its interior.
\end{myassum}

\begin{algorithm}
\caption{Safe-LUCB } \label{algo:Safe-LUCB} 

\begin{algorithmic}[1]
\State \textbf{Pure exploration phase:}
\State \textbf{for} $t = 1,2,\ldots,T'$
\State \quad Randomly choose $x_{t} \in \Dw$ (defined in \eqref{warmup set}) and observe loss $\ell_t = c_t(x_{t})$. 
\State \textbf{end for}
\State \textbf{Safe exploration-exploitation phase:}
\State \textbf{for} $t = T'+1,2,\ldots,T$ \label{line6}
\State \: \: Set $A_{t} = \lambda I+\sum_{\tau=1}^{t-1} x_{\tau}x_{\tau}^\dagger$ and compute $\hat\mu_{t} = A_{t}^{-1}\sum_{\tau=1}^{t-1} \ell_\tau x_\tau$
\State \: \: $\Cc_t = \{v \in \mathbb R^d:\|v-\hat \mu _t\|_{A_t}\leq \beta_t\}$ and $\beta_t$ chosen as in \eqref{eq:beta_t} \label{lineCt}
\State \: \: $\Dts = \{x\in \Dc_0: v^\dagger B x \leq c, \forall v\in \Cc_t\} $
\State \: \: $x_t = \argmin_{x\in \Dts} \min_{v \in \Cc_t} v^\dagger x$ \label{algo:x_t}
\State \quad  Choose $x_t$ and observe loss $\ell_t = c_t(x_t)$.
\State \textbf{end for} \label{line13}
\end{algorithmic}
\end{algorithm}

\subsection{Pure exploration phase}\label{sec:Phase I}
The pure exploration phase of the algorithm runs for rounds $t\in[T']$, where  $T'$ is passed as input to the algorithm. In Section \ref{sec:theory}, we will show how to  appropriately choose its value to guarantee that the cumulative regret is controlled. During this phase, the algorithm selects random actions from a safe subset $\Dw\subset\Dc_0$ that we define next. For every chosen action $x_t$, we observe a loss $\ell_t$. The collected action-loss pairs $(x_t,\ell_t)$ over the $T'$ rounds are used in the second phase to obtain a good estimate of $\mu$. We will see in Section \ref{sec:PhaseII} that this is important since the quality of the estimate of $\mu$ determines our belief of which actions are safe. Now, let us define the safe subset $\Dw$. 

The safe set $\Ds$ is unknown to the learner (since $\mu$ is unknown). However, it can be deduced from the constraint \eqref{eq:constraint} and the boundedness Assumption \ref{assum:bounded} on $\mu$, that the following subset $\Dw\subset\Dc_0$ is safe:
\begin{equation}\label{warmup set}
     \Dw:=
    \{x\in \Dc_0: \max_{\|v\|_2\leq S }v^\dagger B x \leq c\} =\{x\in \Dc_0:  \|B x\|_2 \leq {c}/{S}\}.  
\end{equation}
Note that the set $\Dw$ is only a conservative (inner) approximation of $\Ds$, but this is inevitable, since  the learner has not yet collected enough information on the unknown parameter $\mu$.

In order to make the choice of random actions $x_t,t\in[T']$ concrete, let $X\sim\text{Unif}(\Dw)$ be a d-dimensional random vector uniformly distributed in $\Dw$ according to the probability measure given by the normalized volume in $\Dw$ (recall that $\Dw$ is a convex body by Assumption \ref{assum:nonempty}). During rounds $t\in[T']$, Safe-LUCB chooses safe IID actions $x_t\simiid X$. 
For future reference, we denote the covariance matrix of $X$ by $\Sigma=\E[XX^\dagger]$ and its minimum eigenvalue by
\begin{align}\label{eq:la_-}
\la_-:=\lamin(\Sigma)>0.
\end{align}
\begin{remark} \label{remark:example}
Since $\Dc_0$ is compact with zero in its interior, we can always find $0<\eps\leq C/S$ such that
\begin{align}\label{eq:epsDw}
\widetilde{\Dw} := \{ x\in\R^{d} \,|\, \|Bx\|_2 = \eps \} \subset \Dw.
\end{align}
Thus, an effective way to choose (random) actions $x_t$ during the safe-exploration phase for which an explicit expression for $\lambda_-$ is easily derived, is as follows. For simplicity, we assume $B$ is invertible.
 Let $\eps$ be the largest value $0<\eps\leq c/S$ such that \eqref{eq:epsDw} holds. Then, generate samples $x_t\sim {\rm Unif(\widetilde{\Dw})},  t=1,\ldots,T'$, by choosing $x_t=\eps B^{-1} z_t$, where $z_t$ are iid samples on the unit sphere $\Sc^{d-1}$. Clearly, $\E[z_tz_t^\dagger] = \frac{1}{d} I$. Thus,
%
$
    \Sigma:= \E[x_tx_t^\dagger] = \frac{\eps^2}{d}\left(B^\dagger B\right)^{-1},
    $
from which it follows that
$
    \lambda_- := \lamin(\Sigma)= \frac{\eps}{d\,\lamax\left(B^\dagger B\right)} = \frac{\eps^2}{d\|B\|^2}.
    $
\end{remark}
%


\subsection{Safe exploration-exploitation phase}\label{sec:PhaseII}
We  implement the OFU principle \emph{while respecting the safety constraints}. First, at each $t=T'+1,T'+2\ldots,T$, the algorithm uses the previous action-observation pairs and obtains a $\lambda$-regularized least-squares estimate $\hat\mu_t$ of $\mu$ with regularization parameter $\la>0$ as follows:
\begin{equation}\nn
\hat\mu_t = A_t^{-1}\sum_{\tau=1}^{t-1} \ell_\tau x_\tau,~{\text{where}}~ A_t =\la I+\sum_{\tau = 1}^{t-1} x_{\tau}x_{\tau}^\dagger.
\end{equation}
Then, based on $\hat\mu_t$ the algorithm builds a \emph{confidence set}
    \begin{equation}\label{eq:C_t}
    \Cc_t := \{v \in \mathbb R^d:\|v-\hat \mu _t\|_{A_t}\leq \beta_t\}, 
    \end{equation}
where, $\beta_t$ is chosen according to Theorem \ref{thm:confidence_ball} below (\cite{abbasi2011improved}) to guarantee that $\mu\in\Cc_t$ with high probability.

\begin{theorem}[Confidence Region, \cite{abbasi2011improved}]\label{thm:confidence_ball} 
Let Assumptions \ref{assum:noise} and \ref{assum:bounded} hold. Fix any $\delta\in(0,1)$ and let $\beta_t$ in \eqref{eq:C_t} be chosen as follows,
\begin{align}\label{eq:beta_t}
\beta_t = R\sqrt{d\log\left(\frac{1+(t-1)L^2/\lambda}{\delta}\right)}+\lambda^{1/2}S,\qquad\text{for all}\quad t>0.
\end{align}
Then, with probability at least $1-\delta$, for all $t>0$, it holds that $\mu\in\Cc_t$.
\end{theorem}



The remaining steps of the algorithm also build on existing principles of UCB algorithms. However, here we introduce necessary modifications to account for the safety constraint \eqref{eq:constraint}. Specifically,  we choose the actions with the following two principles.

\noindent{\bf Caution in the face of constraint violation.}
At each round $t$, the algorithm performs conservatively, to ensure that the constraint \eqref{eq:constraint} is satisfied for the chosen action $x_t$. As such, at the beginning of each round $t=T'+1,\ldots,T$, \SUCB~forms the so-called {\it safe decision set}  denoted as $\Dts$:
\begin{equation} \label{eq:safeset}
    \Dts = \{x\in \Dc_0: v^\dagger B x \leq c, \forall v\in \Cc_t \} .
\end{equation}



Recall from Theorem \ref{thm:confidence_ball} that $\mu\in\Cc_t$ with high probability. Thus, $\Dts$  is guaranteed to be a set of safe actions that satisfy \eqref{eq:constraint} with the same probability. On the other hand, note that $\Dts$ is still a conservative inner approximation of $\Cc_t$ (actions in it  are safe for \emph{all} parameter vectors in $\Cc_t$, not only for the true $\mu$).  This (unavoidable) conservative definition of safe decision sets could contribute to the growth of the regret. This is further studied in Section \ref{sec:theory}. 

%

\noindent{\bf Optimism in the face of  uncertainty in cost.} After choosing safe actions randomly at rounds $1,\ldots,T'$, the algorithm creates the safe decision set $\Dts$ at all rounds $t\geq T'+1$, and chooses an action $x_t$ based on the OFU principle. Specifically, a pair $(x_t, \tilde \mu _t)$ is chosen such that
\begin{equation} \label{eq:ofu}
\tilde \mu _t ^ \dagger x_t =  \min_{x\in \Dts , v \in \Cc_t} v^\dagger x.
\end{equation}




%
\section{Regret Analysis of \SUCB~}\label{sec:theory}

\subsection{The regret of safety}\label{sec:regret_of_safety}

In the safe linear bandit problem, the safe set $\Ds$ is not known, since $\mu$ is unknown. Therefore, at each round, the learner chooses actions from a conservative inner approximation of $\Ds$. Intuitively, the better this approximation, the more likely that the optimistic actions of \SUCB~lead to good cumulant regret, ideally of the same order as that of LUCB in the original linear bandit setting.

 A key difference in the analysis of \SUCB ~compared to the classical LUCB is that $x^*$ may not lie within the estimated safe set $\Dts$ at each round.  To see what changes, consider the standard decomposition of the instantaneous regret $r_t,~t=T'+1,\ldots,T$ in two terms as follows (e.g., \cite{Dani08stochasticlinear,abbasi2011improved}):  
\begin{align}
        r_t := \mu^\dagger x_t - \mu^ \dagger x^*
        = \underbrace{\mu^\dagger x_t -\tilde \mu _t ^ \dagger x_t}_{\rm Term~I}\,+\,\underbrace{\tilde \mu _t ^ \dagger x_t - \mu^ \dagger x^*}_{\rm Term~II}, \label{eq:Terms}
    \end{align}
    where, $(\tilde\mu_t,x_t)$ is the optimistic pair, i.e. the solution to the minimization in Step \ref{algo:x_t} of Algorithm \ref{algo:Safe-LUCB}.
    On the one hand, controlling   $\rm Term~I$, is more or less standard and  closely follows previous such bounds on UCB-type algorithms (e.g.,  \cite{abbasi2011improved}); see Appendix \ref{sec:termI} for details.
 On the other hand, controlling $\rm Term~II$, which we call \emph{the regret of safety} is more delicate. This complication lies at the heart of the new formulation with additional safety constraints. When safety constraints are absent, classical LUCB guarantees that $\rm Term~II$ is non-positive. Unfortunately, this is \emph{not} the case here: $x^*$ does \emph{not} necessarily belong to $\Dts$ in \eqref{eq:safeset}, thus $\rm Term~II$ can be positive. This extra regret of safety is the price paid by \SUCB~for choosing safe actions at each round. Our main contribution towards establishing regret guarantees is upper bounding $\rm Term~II$. We show in Section \ref{sec:AT'} that the pure-exploration phase is critical in this direction.

\subsection{Learning the safe set}\label{sec:AT'}
The challenge in controlling the regret of safety 
is that, in general, $\Dts\neq\Ds$.
At a high level, we proceed as follows (see Appendix \ref{sec:termII} for details). First, we relate $\rm Term~II$ with a certain notion of 		``distance'' in the direction of $x^*$ between the estimated set $\Dts$ at rounds $t=T'+1,\ldots,T$ and the true safe set $\Ds$ . Next, we show that this "distance" term can be controlled by appropriately lower bounding the minimum eigenvalue $\lamin(A_t)$ of the Gram matrix $A_t$. Due to the interdependency of the actions $x_t$, it is difficult to directly establish such a lower bound for each round $t$. Instead, we use that $\lamin(A_t)\geq \lamin(A_{T'+1}),~t\geq T'+1$ and we are able to bound $\lamin(A_{T'+1})$ thanks to the pure exploration phase of \SUCB~. 
Hence, the pure exploration phase guarantees that $\Dts$ is a sufficiently good approximation to the true $\Ds$ once the exploration-exploitation phase begins. 
%
%

\begin{lemma}\label{lemm:HT'}
 Let $A_{T'+1} = \la I+\sum_{t=1}^{T'} x_tx_t^\dagger$ be the Gram matrix corresponding to the first $T'$ actions of Safe-LUCB (pure-exploration phase). Recall the definition of $\la_{-}$ in \eqref{eq:la_-}. Then, for any $\delta\in(0,1)$, it holds with probability at least $1-\delta$,
\begin{equation}
    \lamin(A_{T'+1})\geq \la+\frac{\lambda_-T'}{2},
\end{equation}
provided that  $T' \geq t_{\delta}:=\frac{8L^2}{\lambda_-}\log(\frac{d}{\delta})$.
\end{lemma}
The proof of the lemma and technical details relating the result to a desired bound on $\rm Term~II$ are deferred to Appendixes \ref{sec:proofHT'} and \ref{sec:termII}, respectively. 

\subsection{Problem dependent upper bound}\label{sec:Delta>0}

In this section, we present a problem-dependent upper bound on the regret of \SUCB~in terms of the following critical parameter, which we call the \emph{safety gap}:
\begin{align}
\Delta := c - \mu ^\dagger B x^*.
\end{align}
Note that $\Delta \geq 0$. In this section, we assume that $\Delta$ is known to the learner. The next lemma shows that if $\Delta>0$ \footnote{We remark that the case $\Delta>0$ studied here is somewhat reminiscent of the assumption $\alpha r_\ell>0$ in \cite{vanroy}.
}
, then choosing $T'=\Oc(\log{T})$ guarantees that $x^*\in\Dts$ for all $t=T'+1,\ldots,T$.

\begin{lemma}[$x^*\in\Dts$]\label{lem:x^*in}
Let Assumptions \ref{assum:noise}, \ref{assum:bounded} and \ref{assum:nonempty} hold. Fix any $\delta\in(0,1)$ and assume a positive safety gap $\Delta>0$. 
Initialize \SUCB~with (recall the definition of $t_\delta$ in Lemma \ref{lemm:HT'})
\begin{align}
T'\geq T_\Delta:=  \bigg( \frac{8L^2\|B\|^2\beta_{T}^2}{\lambda_-\,\Delta ^2}-\frac{2\la}{\la_-} \bigg) \,\vee\, t_\delta.\label{eq:T0}
\end{align}
 Then, with probability at least $1-\delta$, for all $t=T'+1,\ldots,T$ it holds that $x^*\in\Dts$.
\end{lemma}

In light of our discussion in Sections \ref{sec:regret_of_safety} and \ref{sec:AT'}, once we have established that  $x^*\in\Dts$ for $t=T'+1,\ldots,T$, the regret of safety becomes nonpositive and we can show that the algorithm performs just like classical LUCB during the exploration-exploitation phase \footnote{Our simulation results in Appendix \ref{sec:sim} emphasize the critical role of a sufficiently long pure exploration phase by \SUCB~as suggested by Lemma \ref{lem:x^*in}. Specifically, Figure \ref{fig:regret,polytope} depicts an instance where \emph{no} exploration leads to significantly worse order of regret.}. This is formalized in Theorem \ref{thm:regretdeltanotzero} showing that when $\Delta>0$ (and is known), then the regret of \SUCB~is $\Otilde(\sqrt{T})$.

\begin{theorem}[Problem-dependent bound; $\Delta>0$]\label{thm:regretdeltanotzero}
Let the same assumptions as in Lemma \ref{lem:x^*in} hold. Initialize \SUCB~with $T'\geq T_\Delta$ specified in \eqref{eq:T0}. Then, for $T\geq T'$, with probability at least $1-2\delta$, the cumulative regret of \SUCB~satisfies
\begin{align}\label{eq:bound_Delta>0} 
    R_T\leq 2T'+2\beta_T\sqrt{2d \,(T-T') \,\log \left(\frac{2TL^2}{d(\lambda_-T'+2\la)}\right)} \,\,.
\end{align}
Specifically, choosing $T'=T_{\Delta}$ guarantees cumulant regret $\Oc(T^{1/2}\log{T})$.
\end{theorem}

The bound in \eqref{eq:bound_Delta>0} is a contribution of two terms. The first one is a trivial bound on the regret of the exploration-only phase of \SUCB~and is proportional to its duration $T'$.
Thanks to Lemma \ref{lem:x^*in} the duration of the  exploration phase is limited to $T_\Delta$ rounds and $T_\Delta$ is (at most) logarithmic in the total number of rounds $T$. 
Thus, the first summand in \eqref{eq:bound_Delta>0} contributes only $\Oc(\log{T})$ in the total regret. Note, however, that $T_\Delta$ grows larger as the normalized safety gap $\Delta/\|B\|$ becomes smaller. The second summand in \eqref{eq:bound_Delta>0} contributes $\Oc(T^{1/2}\log{T})$ and bounds the cumulant regret of the exploration-exploitation phase, which takes the bulk of the algorithm. More specifically, it bounds the contribution of $\rm Term~I$ in \eqref{eq:Terms} since the $\rm Term~II$ is zeroed out once $x^*\in\Dts$ thanks to Lemma \ref{lem:x^*in}. 
Finally, note that Theorem \ref{thm:regretdeltanotzero} requires the total number of rounds $T$ to be large enough for the desired regret performance. This is the price paid for the extra safety constraints compared to the performance of the classical LUCB in the original linear bandit setting.
We remark that existing lower bounds for the simpler problem without safety constraints (e.g. \cite{Tsitsiklis,Dani08stochasticlinear}), show that the regret $\Otilde(\sqrt{Td})$ of Theorem \ref{thm:regretdeltanotzero} cannot be improved modulo logarithmic factors.
The proofs of Lemma \ref{lem:x^*in} and Theorem \ref{thm:regretdeltanotzero} are in Appendix \ref{proof:thm2}. 


\subsection{General upper bound}
\vspace{-5pt}
We now extend the results of Section \ref{sec:Delta>0} to instances where the safety gap is zero, i.e. $\Delta=0$. In this case, we cannot guarantee an exploration phase that results in $x^*\in\Dts, t>T'$ in a reasonable time length $T'$. Thus, the regret of safety is not necessarily non-positive and it is unclear whether a sub-linear cumulant regret is possible.

Theorem \ref{thm:worst-case} shows that \SUCB~achieves regret $\Otilde(T^{2/3})$ when $\Delta=0$. Note that this (worst-case) bound is also applicable when the safety gap is unknown to the learner. While it is significantly worse than the performance guaranteed by Theorem \ref{thm:regretdeltanotzero}, it proves that \SUCB~always leads to $R_T/T\rightarrow 0$ as $T$ grows large. The proof is deferred to Appendix \ref{proof:thm2}.

\begin{theorem}[General bound: worst-case]\label{thm:worst-case}
Suppose Assumptions \ref{assum:noise}, \ref{assum:bounded} and \ref{assum:nonempty} hold. Fix any $\delta\in(0,0.5)$. Initialize \SUCB~with $T'\geq t_\delta$ specified in Lemma \ref{lemm:HT'}. Then, with probability at least $1-2\delta$ the cumulative regret $R_T$ of \SUCB~for 
$T\geq T'$ satisfies
\begin{align}\label{eq:Delta=0}
    R_T \leq 2T'+2\beta_T\sqrt{2d(T-T') \log \left(\frac{2TL^2}{d(\lambda_-T'+2\la)}\right)}+\frac{2\sqrt{2}\|B\|L\beta_T (T-T')}{c\sqrt{\lambda_-T'+2\la}}.
\end{align}
\mbox{Specifically, choosing ${T'=T_0:=\Big(\frac{\|B\|L \beta_T T}{c \sqrt{2\lambda_-}}\Big)^{\frac{2}{3}}\vee t_{\delta}}$\,, guarantees regret $\Oc(T^{2/3}\log T)$.}
\end{theorem} 

Compared to Theorem \ref{thm:regretdeltanotzero}, the bound in \eqref{eq:Delta=0} is now comprised of three terms. The first one captures again the exploration-only phase and is linear in its duration $T'$. However, note that $T'$ is now $\Oc(T^{2/3}\log T)$, i.e., of the same order as the total bound. The second term bounds the total contribution of $\rm Term~I$ of the exploration-exploitation phase. As usual, its order is $\Otilde(T^{1/2})$. Finally, the additional third term bounds the regret of safety and is of the same order as that of the first term.

\vspace{-10pt}
\section{Unknown Safety Gap}\label{sec:GSUCB}
\vspace{-5pt}
In Section \ref{sec:Delta>0} we showed that when the safety gap $\Delta>0$, then \SUCB~achieves good regret performance $\Otilde(\sqrt{T})$. However, this requires that the value of $\Delta$, or at least a (non-trivial) lower bound on it, be known to the learner so that $T'$ is initialized appropriately according to Lemma \ref{lem:x^*in}. This requirement  might be restrictive in certain applications. When that is the case, one option is to run \SUCB ~with a choice of $T'$ as suggested by Theorem \ref{thm:worst-case}, but this could result in an unnecessarily long pure exploration period (during which regret grows linearly). Here, we present an alternative. Specifically, we propose a  variation of \SUCB ~refered to as \emph{generalized safe linear upper confidence bound} (GSLUCB). The key idea behind GSLUCB is to build a lower confidence bound $\Delta_t$ for the safety gap $\Delta$ and calculate the   length of the pure exploration phase associated with $\Delta_t$, denoted as $T'_t$. This allows the learner to stop the pure exploration phase at  round $t$ such that condition $t \leq T'_{t-1}$ has been met. While we do not provide a separate regret analysis for GSLUCB, it is clear that the worst case regret performance would match that of \SUCB ~with $\Delta = 0$. However, our numerical experiment highlights the improvements that GSLUCB can provide for the cases where $\Delta \neq 0$. We give a full explanation of GSLUCB, including how we calculate the lower confidence bound $\Delta_t$,  in Appendix \ref{algo:GSLUCB}.

Figure \ref{fig:regret,karmed} compares the average per-step regret of 1) \SUCB  ~with knowledge of $\Delta$; 2) \SUCB  ~without knowledge of $\Delta$ (hence, assuming $\Delta =0$); 3) GSLUCB without knowledge of $\Delta$, in a simplified setting of $K$-armed linear bandits with strictly positive safety gap (see Appendix \ref{sec:contextual}). The details on the parameters of the simulations are deferred to Appendix \ref{sec:sim}.

%
%

\begin{algorithm}
\caption{GSLUCB} \label{algo:GapUCB}  

\begin{algorithmic}[1]
\State \textbf{Pure exploration phase:}


\State $t \leftarrow  1$ , $T'_{0} = T_0$ 

\State \textbf{while} $\left( t\leq \min\left(T'_{t-1}, T_0 \right) \right)$
\State \quad Randomly choose $x_{t} \in \Dc^w$ and observe loss $\ell_t = c_t(x_{t})$.
\State  \quad  $\Delta_t$ = Lower confidence bound on $\Delta$  at round $t$
\State \quad \textbf{if} $\Delta_t>0 $ \textbf{then}  $T'_t = T_{\Delta_t}$ \label{T'delta}
\State \quad \textbf{else} $T'_t  = T_0$
\State \quad \textbf{end if}
\State  \quad $t \leftarrow  t+1$
\State \textbf{end while}
\State \textbf{Safe exploration exploitation phase:} Lines \ref{line6} - \ref{line13} of \SUCB ~for all remaining rounds.
\end{algorithmic}
\end{algorithm}


%
%

%

\begin{figure}
     \centering
     \begin{subfigure}[b]{0.48\textwidth}
         \centering
         \includegraphics[width=\textwidth]{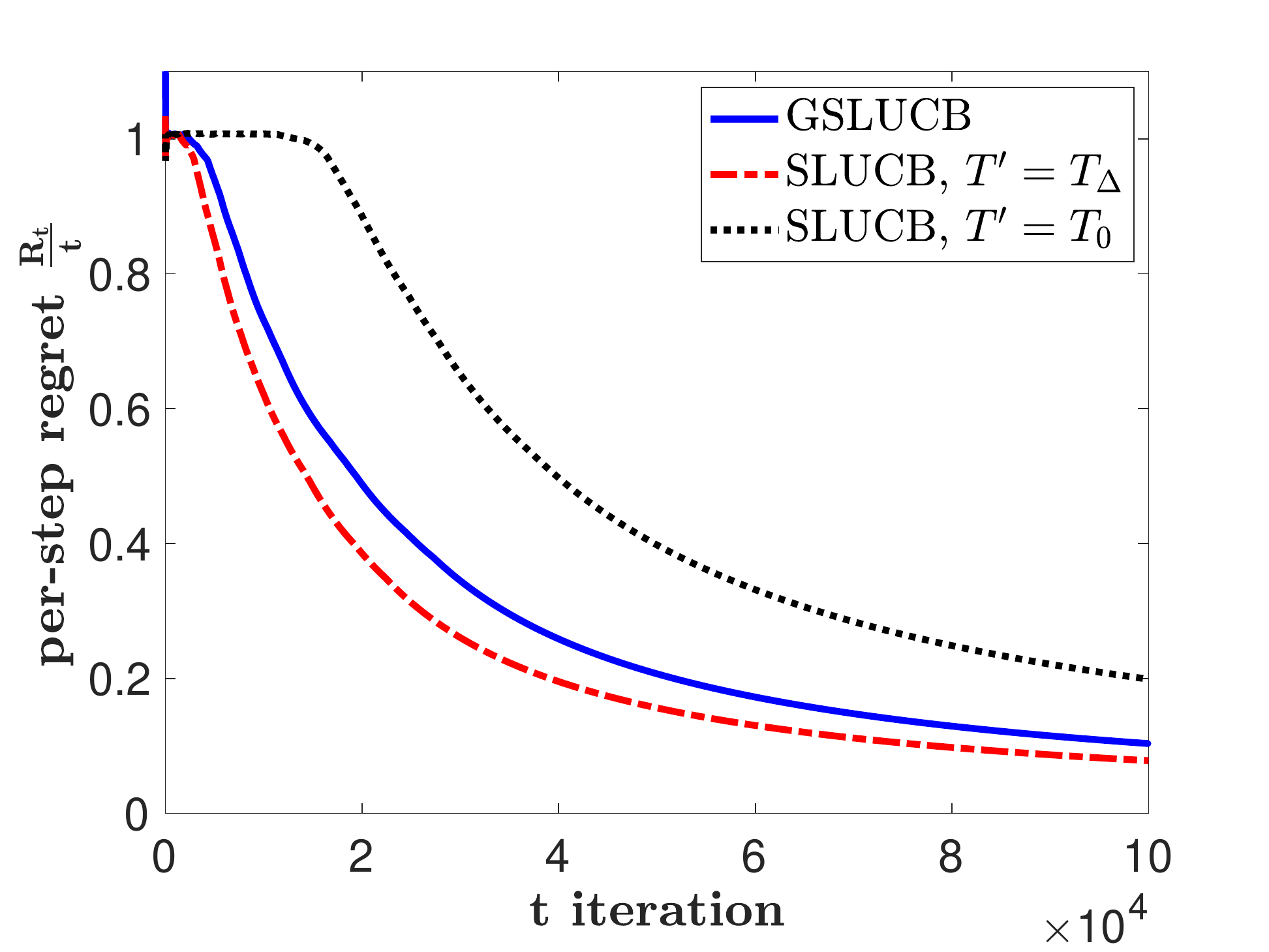}
         \caption{Average per-step regret of \SUCB ~and GSLUCB with a decision set of $K$ arms.}
         \label{fig:regret,karmed}
     \end{subfigure}
    \hfill 
     \begin{subfigure}[b]{0.48\textwidth}
         \centering
         \includegraphics[width=\textwidth]{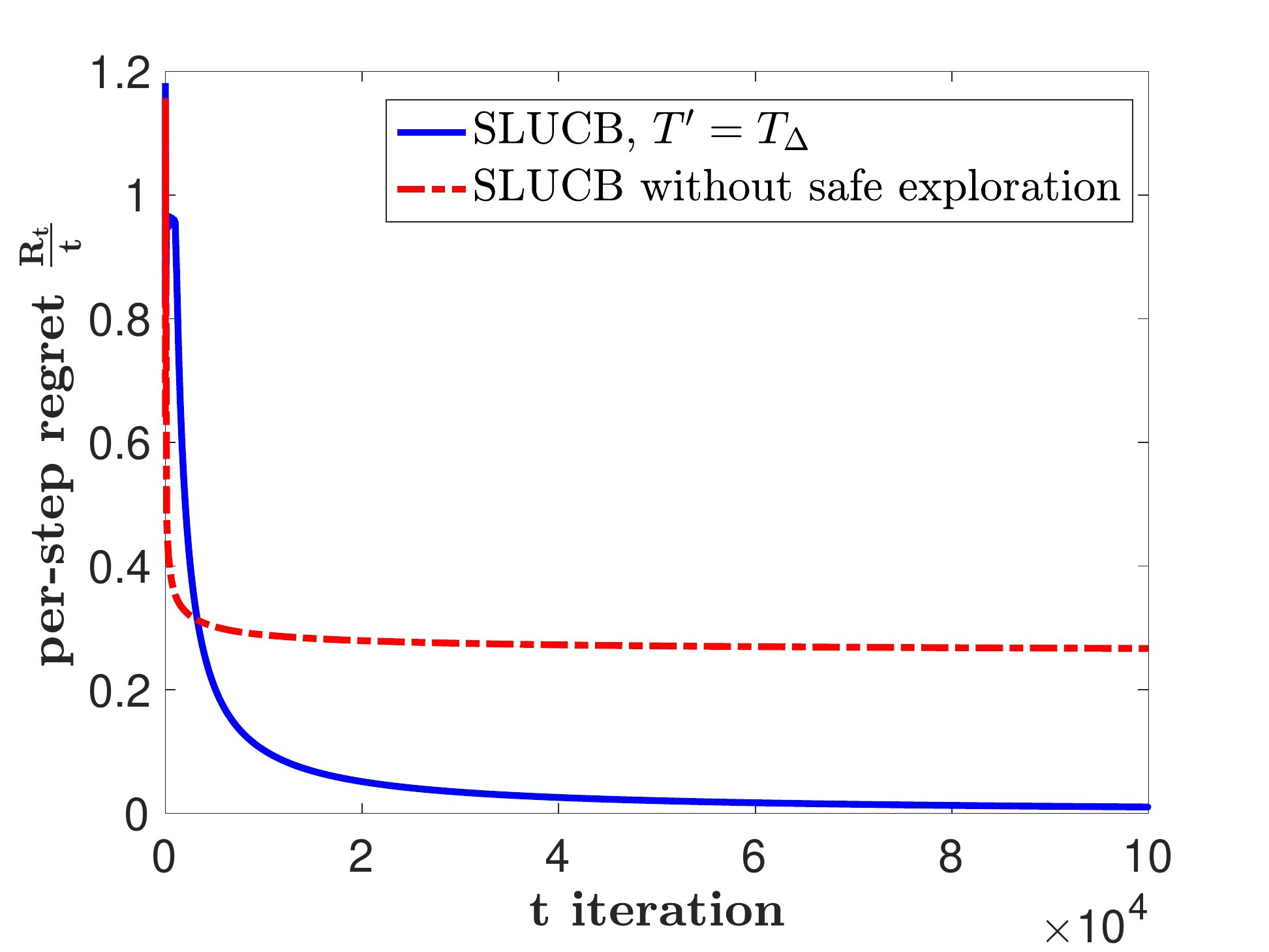}
          \caption{Per-step regret of \SUCB ~with and without pure exploration phase.} \label{fig:regret,polytope}
     \end{subfigure}
   \caption{Simulation of per-step regret.}
   \label{fig:regret}
\end{figure}

\vspace{-15pt}
\section{Conclusions}
\vspace{-10pt}
We have formulated a linear stochastic bandit problem with safety constraints that depend linearly on the unknown problem parameter $\mu$. While simplified, the model captures the  additional complexity introduced in the problem by the requirement that chosen actions belong to an unknown safe set. As such, it allows us to quantify tradeoffs between learning the safe set and minimizing the regret. Specifically, we  propose \SUCB~which is comprised of two phases: (i) a pure-exploration phase that speeds up learning the safe set; (ii) a safe exploration-exploitation phase that optimizes minimizing the regret. Our analysis suggests that the safety gap $\Delta$ plays a critical role. When $\Delta>0$ we show how to achieve regret $\Otilde(\sqrt{T})$ as in the classical linear bandit setting.
However,  when $\Delta=0$, the regret of \SUCB~is $\Otilde(T^{2/3})$.  It is an interesting open problem to establish lower bounds for an arbitrary policy that accounts for the safety constraints. Our analysis of \SUCB~suggests that $\Delta=0$ is a worst-case scenario, but it remains open whether the $\Otilde(T^{2/3})$ regret bound can be improved in that case. Natural extensions of the problem setting to multiple constraints and generalized linear bandits (possibly with generalized linear constraints) might also be of interest.

\newpage
\bibliographystyle{unsrt}
\bibliography{bibfile}

\newpage
\appendix

\section{Proof of Lemma \ref{lemm:HT'}} \label{sec:proofHT'}
In order to bound the minimum eigenvalue of the Gram matrix at round $T'+1$, we use the Matrix Chernoff Inequality \cite[Thm.~5.1.1]{tropp2015introduction}.
\begin{theorem}[Matrix Chernoff Inequality,~\cite{tropp2015introduction}]\label{thm:chernoff}
Consider a finite sequence $\{X_k\}$ of independent, random, symmetric matrices in $\mathbb{R}^d$. Assume that $\lamin(X_k) \geq 0$ and $\lamax(X_k) \leq L$ for each index $k$. Introduce the random matrix $Y = \sum _k X_k$. Let $\mu_{min}$ denote the minimum eigenvalue of the expectation $\E[Y]$,
\begin{equation}
    \mu_{\rm min} = \lamin\left(\E[Y]\right)=\lamin\left(\sum _k E[X_k]\right). \nonumber
\end{equation}
Then, for any $\eps\in(0,1)$, it holds,
\begin{equation}
    \Pr\left( \lamin(Y)\leq\epsilon\mu_{\rm min}\right)\leq d\cdot \exp\left(-(1-\epsilon)^2\frac{\mu_{\rm min}}{2L}\right). \nonumber
\end{equation}

\end{theorem}

\begin{proof}[Proof of Lemma \ref{lemm:HT'}.]~
Let $X_t = x_t x_t^\dagger$ for $t \in [T']$, such that each $X_t$ is a symmetric matrix with $\lamin(X_t)\geq 0$ and $\lamax(X_t)\leq L^2$. In this notation, $
    A_{T'+1} = \la I +\sum_{t = 1} ^{T'} X_t.$ 
In order to apply Theorem \ref{thm:chernoff}, we compute:
\begin{align}
    \mu_{\rm min}  &:= \lamin\left(\sum _{t=1}^{T'} \E[X_t]\right)\nonumber
    =\lamin\left(\sum _{t=1}^{T'} \E[x_t x_t^\dagger]\right)\nonumber
    =\lamin\left(T' \Sigma \right)\nonumber
    =\lambda_- T'. \nonumber
\end{align}
Thus, the theorem implies the following for any $\epsilon\in[0,1)$:
\begin{align} \label{eq:chernoff_in_lemma}
    \Pr\left[ \lamin(\sum_{t = 1} ^{T'} X_t)\leq\epsilon\lambda_-T'\right]\leq d\cdot \exp\left(-(1-\epsilon)^2\frac{\lambda_-T'}{2L^2}\right). 
\end{align}
To complete the proof of the lemma, simply choose $\epsilon = 0.5$ (say) and $T' \geq \frac{8L^2}{\lambda_-}\log(\frac{d}{\delta})$ in  \eqref{eq:chernoff_in_lemma}. This gives
$
    \Pr\left[ \lamin(A_{T'+1})\geq \la+\frac{\lambda_-T'}{2}\right]\geq 1-\delta,
$ as desired.
\end{proof}

\section{Proof of Theorems \ref{thm:regretdeltanotzero} and \ref{thm:worst-case}} \label{proof:thm2}
In this section, we present the proofs of Theorems \ref{thm:regretdeltanotzero} and \ref{thm:worst-case}. 


\subsection{Preliminaries}
\paragraph{Conditioning on $\mu\in\Cc_t,~\forall t> 0$.}  
Consider the event 
\begin{align}\label{eq:conditioned}
\Ec:=\{\mu\in \Cc_t,~\forall t> 0\},
\end{align}
 that $\mu$ is inside the confidence region for all times $t$. By Theorem \ref{thm:confidence_ball} the event holds with probability $1-\delta$. Onwards, we condition on this event, and we make repeated use of the fact that $\mu\in\Cc_t$ for all $t>0$, without further explicit reference.

\paragraph{Decomposing the regret in two terms.} Recall the decomposition of the instantaneous regret in two terms in \eqref{eq:Terms} as follows: 
%
%
\begin{align}
        r_t &= \mu^\dagger x_t - \mu^ \dagger x^*
        = \underbrace{\mu^\dagger x_t -\tilde \mu _t ^ \dagger x_t}_{\text{Term I}}+\underbrace{\tilde \mu _t ^ \dagger x_t - \mu^ \dagger x^*}_{\text{Term II}}. \label{eq:Terms_App}
    \end{align}
    
 As discussed in Section \ref{sec:regret_of_safety}, we control the two terms separately.
 
%


\subsection{Bounding Term I} \label{sec:termI}
The results in this subsection are by now rather standard in the literature (see for example \cite[]{abbasi2011improved}). We provide the necessary details for completeness.

We start with the following chain of inequalities, that hold for all $t\geq T'+1$:
\begin{align}
\text{Term I}&:= \mu^\dagger x_t -\tilde \mu _t ^ \dagger x_t
        =(\mu^\dagger x_t -\hat \mu _t ^ \dagger x_t)+(\hat \mu _t ^ \dagger x_t - \tilde \mu _t ^ \dagger x_t)\nn\\
        &\leq \|\mu-\hat \mu _t\|_{A_t}\|x_t\|_{A_t^{-1}}+\|\hat\mu_t-\tilde \mu _t\|_{A_t}\|x_t\|_{A_t^{-1}}\nn\\
        &\leq 2\beta_t \|x_t\|_{A_t^{-1}}\label{eq:TermIa}.
    \end{align}
The last inequality \eqref{eq:TermIa} follows from Theorem \ref{thm:confidence_ball} and the fact that $\mu$ and $\tilde\mu_t \in\Cc_t$. Recall, from Assumption \ref{assum:bounded}, the trivial bound on the instantaneous regret $$r_t=\mu^\dagger x_t - \mu^\dagger x^* \leq 2.$$ 
Thus, we conclude with the following
\begin{align}
    \text{Term I} 
    &\leq 2\min\big(\beta_t \|x_t\|_{A_t^{-1}},1\big). \label{eq:TermI_b}
\end{align}

The next lemma bounds the total contribution of the (squared) terms in \eqref{eq:TermIa} across the entire horizon $t=T'+1,\ldots,T$.
\begin{lemma}[Term I]\label{lemm:sum of width}
    Let Assumptions \ref{assum:noise} and \ref{assum:bounded} hold. Fix any $\delta\in(0,0.5)$ and assume that $T'$ is such that
$
T'\geq \frac{8L^2}{\lambda_-} \log\Big(\frac{d}{\delta}\Big)
$. Then,  with probability at least $1-\delta$, it holds
    \begin{equation}
        \sum_{t=T'+1}^T \min \left(\|x_t\|_{A_t^{-1}}^2,1\right)  \leq 2d \log \left(\frac{2TL^2}{d(2\la+\lambda_-T')}\right). \nonumber
    \end{equation}
    Thus, with probability at least $1-2\delta$, it holds
      \begin{equation}
        \sum_{t=T'+1}^T \big(\mu^\dagger x_t -\tilde \mu _t ^ \dagger x_t \big) \leq 2\beta_T\sqrt{2d\,(T-T')\, \log \left(\frac{2TL^2}{d\,(2\la+\lambda_-T')}\right)} \label{eq:TermI_app}.
    \end{equation}  
    \end{lemma}

    \begin{proof} The proof is mostly adapted from \cite[Lem.~9]{Dani08stochasticlinear} but we also exploit the bound on $\lamin(A_{T'+1})$ thanks to Lemma \ref{lemm:HT'}. We present the details for the reader's convenience.
     
    With probability at least $1-\delta$, we find that for all $t\geq T'+1$:
    \begin{align}
        \det(A_{t+1}) &= \det(A_t + x_t x_t^\dagger) = \det(A_t)\det(I+(A_t^{-\frac{1}{2}}x_t)(A_t^{-\frac{1}{2}}x_t)^\dagger) = \det(A_t)(1+\|x_t\|_{A_t^{-1}}^2) \nonumber\\
        &= \ldots = \det(A_{T'+1})\prod_{\tau = T'+1}^t (1+\|x_\tau\|_{A_\tau^{-1}}^2) \nn\\
        &\geq \left(\la+\frac{\lambda_-T'}{2}\right)^d \prod_{\tau = T'+1}^t(1+\|x_\tau\|_{A_\tau^{-1}}^2),\nn
    \end{align}
    where the last inequality follows form Lemma \ref{lemm:HT'} and the fact that $\det(A)=\prod_{i=1}^d \lambda_i(A) \geq (\lamin(A))^d$. Furthermore, by the AM-GM inequality applied to the eigenvalues of $A_{t+1}$, if holds 
    $$\det(A_{t+1}) = \prod_{i=1}^d \lambda_i(A_{t+1})\leq \left(\frac{tL^2}{d}\right)^d,$$ 
    where we also used the fact that $\|x_t\|_2\leq L$ for all $t$.
    These combined yield,    
    \begin{align*}
        \prod_{\tau = T'+1}^t(1+\|x_\tau\|_{A_\tau^{-1}}^2) \leq \left(\frac{2tL^2}{d(2\la+\lambda_-T')}\right)^d.
    \end{align*}
Next,  using the fact that for any $0\leq y \leq 1$, $\log(1+y)\geq y/2$, we have 
   \begin{align*}
       \sum_{t=T'+1}^T \min \left(\|x_t\|_{A_t^{-1}}^2,1\right) &\leq 2\sum_{t=T'+1}^T\log \left(\|x_t\|_{A_t^{-1}}^2+1\right) =2\log\Big(\prod_{t=T'+1}^T \big(\|x_t\|_{A_t^{-1}}^2+1\big)\,\Big)\\
       &\leq 2d \log \left(\frac{2TL^2}{d(2\la+\lambda_-T')}\right).
   \end{align*}
   It remains to prove \eqref{eq:TermI_app}. Recall from \eqref{eq:TermI_b} that for any $T'< t \leq T$, with probability at least $1-\delta$ (note that we have conditioned in the event $\Ec$ in \eqref{eq:conditioned}), 
   $$
   \big(\mu^\dagger x_t -\tilde \mu _t ^ \dagger x_t \big)\leq2\min\big(\beta_t \|x_t\|_{A_t^{-1}},1\big) \leq 2\beta_T\min\big(\|x_t\|_{A_t^{-1}},1\big),
   $$
   where for the inequality we have used the fact that $\beta_t\leq \beta_T$ (and assumed for simplicity that $T$ large enough such that $\beta_T>1$). Thus, the desired bound in \eqref{eq:TermI_app} follows from applying Cauchy-Schwartz inequality to the above. 
    \end{proof}

    

\subsection{Bounding Term II} \label{sec:termII}
As discussed in Section \ref{sec:AT'}, the challenge in bounding Term II in \eqref{eq:Terms} is that , in general, $\Dts\neq\Ds$, so $x^*$ might not belong in $\Dts$. Bounding Term II amounts to bounding a certain "distance" of the set $\Dts$ from the set $\Dc_0$. In order to accomplish this task, we proceed as follows. First, we define a shrunk version $\Dtss$ of $\Dts$, for which we have a more convenient characterization, compared to the original $\Dtss$. Then, we select the point $z_t$ in $\Dtss$ that is in the direction of $x^*$ and is as close to it as possible. Finally, we are able to bound the distance of $z_t$ to $x^*$.

\paragraph{A shrunk safe region $\Dtss$.}  Consider an enlarged confidence region $\tilde {\Cc_t}$ centered at $\mu$ defined as follows:
\begin{equation} \label{eq:ctilde}
    \tilde {\Cc_t} := \{v \in \mathbb{R}^d:\|v-\mu\|_{A_t} \leq 2 \beta_{t}\}  \supseteq \Cc_t.
\end{equation}
The inclusion property above holds since  $\mu\in\Cc_t$, and, by triangle inequality, for all $v\in \Cc_t$, one has that
$
\|v-\mu\|_{A_t} \leq \|v-\hat\mu_t\|_{A_t}+\|\hat\mu_t-\mu\|_{A_t} \leq 2\beta_t.
$


The definition of the enlarged confidence region in \eqref{eq:ctilde} naturally leads to the definition of a corresponding shrunk safe decision set $\Dtss$. Namely, let
\begin{align} \label{eq:Din}
    \Dtss &:= \{ x \in \Dc_0: v^\dagger Bx \leq c, \: \forall v \in \tilde {\Cc_t}\} \nonumber
 = \{x \in \Dc_0: \max_{v\in \tilde {\Cc_t} } v^\dagger Bx\leq c\}\nonumber\\
  &=\{x \in \Dc_0: \mu^\dagger Bx+2\beta_
  t\|Bx\|_{A_t^{-1}}\leq c\},
\end{align}
and observe that $\Dtss \subseteq \Dts$.
Note here that since by Assumption \ref{assum:nonempty} zero is in the interior of $\Dc_0$, the sets $\tilde{\Dts}$ and  $\Dts$ have a nonempty interior.

\paragraph{A point $z_t\in\Dtss$ close to $x^*$.} Let $z_t$ be a vector in the direction of $x^*$ that belongs in $\Dtss$ and is closest to $x^*$. Formally, $z_t := \alpha_t x^*$, where 
\begin{align}
\alpha_t:= \max\Big\{\alpha\in[0,1]\,|\, z_t = \alpha x^*\in\Dtss\Big\}.\nn
\end{align}
Since both $0$ and $x^*\in\Dc_0$, and, $\Dc_0$ is convex by assumption, it follows in view of \eqref{eq:Din} that
\begin{align}
\alpha_t:= \max\left\{\alpha\in[0,1]\,\Huge|\, \alpha\cdot\Big(\mu^\dagger Bx^*+2\beta_
  t\|Bx^*\|_{A_t^{-1}}\Big) \leq c \right\}.\label{eq:alphaa_0}
\end{align}
Recall that $C>0$, thus \eqref{eq:alphaa_0} can be simplified to the following:
\begin{equation}\label{eq:alphaa}
    \alpha_t=
    \begin{cases}
      1 &, \text{if}\quad \mu^\dagger B x^*+2\beta_t\|Bx^*\|_{A_t^{-1}}\leq c, \\
      \min\left(\frac{c}{\mu^\dagger B x^*+2\beta_t\|Bx^*\|_{A_t^{-1}}},1\right) &, \text{otherwise}.
    \end{cases}
\end{equation}

\paragraph{Bounding $\rm Term~II$ in terms of $\alpha_t$.} Due to the fact that $\Dtss \subseteq \Dts$, it holds that $z_t \in \Dts$. Using this, and optimality of $(\tilde\mu, x_t)$ in the minimization in Step \ref{algo:x_t} of Algorithm \ref{algo:Safe-LUCB}, we can bound Term II as follows:
\begin{align}
\text{Term II}&:=\tilde \mu _t ^ \dagger x_t - \mu^ \dagger x^*\nn\\
&\leq \mu ^ \dagger z_t - \mu^ \dagger x^* = \alpha_t\mu ^ \dagger x^* - \mu^ \dagger x^* \nn\\
&\leq |a_t-1|\,|\mu^\dagger x^*|\nn\\
&\leq |a_t-1| = (1-\alpha_t).\label{eq:TermIIa}
\end{align}
The inequality in the last line uses Assumption \ref{assum:bounded}. For the last equality recall that $\alpha_t\in[0,1]$

To proceed further from \eqref{eq:TermIIa} we consider separately the two cases $\Delta>0$ and $\Delta=0$ that lead to Theorems \ref{thm:regretdeltanotzero} and \ref{thm:worst-case}, respectively. 

%

\subsubsection{Bound for the case $\Delta>0$}
 Here, assuming that $\Delta>0$, we prove that if the duration $T'$ of the pure exploration phase of \SUCB~ is chosen appropriately, then $\alpha_t=1$, and equivalently, $x^*\in\Dts$. The precise statement is given in Lemma \ref{lem:xin_app} below, which is a restatement of Lemma \ref{lem:x^*in}, given here for the reader's convenience.


\begin{lemma}[$\Delta>0~\implies~x^*\in\Dts$]\label{lem:xin_app}
Let Assumptions \ref{assum:noise}, \ref{assum:bounded} and \ref{assum:nonempty} hold for all $t \in [T]$. Fix any $\delta\in(0,0.5)$ and assume a positive safety gap $\Delta>0$. 
Initialize \SUCB~with 
\begin{align}
T'\geq \bigg( \frac{8L^2\|B\|^2\beta_{T}^2}{\lambda_-\,\Delta ^2}-\frac{2\la}{\la_-} \bigg) \,\vee\, t_\delta.\label{eq:T0_app}
\end{align}
 Then, with probability at least $1-2\delta$, for all $t=T'+1,\ldots,T$ it holds that $${\rm Term~II}:=\tilde \mu _t ^ \dagger x_t - \mu^ \dagger x^*\leq0.$$ 
 Thus, with the same probability
 \begin{align}\label{eq:Delta=0_TermII}
 \sum_{t=T'+1}^{T} (\tilde \mu _t ^ \dagger x_t- \mu^ \dagger x^*) \leq 0.
 \end{align}
\end{lemma}

\begin{proof}
Recall from \eqref{eq:TermIIa}, that for any $T'< t \leq T$, with probability at least $1-\delta$ (note that we have conditioned in the event $\Ec$ in \eqref{eq:conditioned}), $\text{Term II}=1-\alpha_t$. Thus, in view of \eqref{eq:alphaa}, it suffices to prove that for any $T'< t \leq T$, with probability at least $1-\delta$, it holds $\alpha_t=1$, or equivalently,
\begin{align}\label{eq:2show}
\mu^\dagger B x^*+2\beta_t\|Bx^*\|_{A_t^{-1}}\leq c ~~\Leftrightarrow~~ \beta_t\|Bx^*\|_{A_t^{-1}}\leq \Delta/2.
\end{align}
 For any $T'< t \leq T$, we have
\begin{align}\label{eq:bb_00}
\beta_t\|Bx^*\|_{A_{t}^{-1}} \leq \frac{\beta_{t}\|Bx^*\|_2}{\sqrt{\lamin(A_{t})}}\leq
 \frac{\beta_{T}\|Bx^*\|_2}{\sqrt{\lamin(A_{T'+1})}} \leq
 \frac{\beta_{T}\|B\| L}{\sqrt{\lamin(A_{T'+1})}},
\end{align}
where, in the second inequality we used $\beta_t\leq \beta_T$ and $\lamin(A_{t})\geq \lamin(A_{T'+1})$, and in the last inequality we used Assumption \ref{assum:bounded}. Next, since  $t_\delta\leq T'$, we may apply Lemma \ref{lemm:HT'}  to find from \eqref{eq:bb_00},  that for all $T'+1\leq t\leq T$, with probability at least $1-\delta$:
\begin{equation}
\beta_t\|Bx^*\|_{A_{t}^{-1}} \leq \frac{\sqrt{2}\|B\|L\beta_{T}}{\sqrt{2\la+\lambda_-{T'}}}. \label{eq:bb_11}
\end{equation}
To complete the proof of the lemma note that the assumption $T'\geq \frac{8\|B\|^2L^2\beta_{T}^2}{\lambda_-\Delta ^2}-\frac{2\la}{\la_-}$ when combined with \eqref{eq:bb_11}, it guarantees \eqref{eq:2show}, as desired.
\end{proof}

\begin{remark}\label{remark:modified Confidence region}
We remark on a simple tweak in the algorithm that results in a constant $T'$ (i.e., independent of $T$) in Lemma \ref{lem:xin_app}. However, this does \emph{not} change the final order of regret bound in Theorem \ref{thm:regretdeltanotzero}.
In particular, we modify \SUCB ~ to use
the nested (as is called in \cite{vanroy}) confidence region $\Bc_t = \cap_{\tau=1}^t \Cc_{\tau}$ at round $t$ such that $\ldots  \subseteq\Bc_{t+1} \subseteq \Bc_t \subseteq \Bc_{t-1} \subseteq \ldots $. According to Theorem \ref{thm:confidence_ball}, it is guaranteed that for all $t>0$, $\mu \in \Bc_t$, with high probability. Applying these nested confidence regions in creating safe sets, results in $\dots \subseteq \mathcal{D}_{t-1}^{s}\subseteq \mathcal{D}_{t}^{s} \subseteq \mathcal{D}_{t+1}^{s} \subseteq \dots$. Thanks to this, it is now guaranteed that once $x^* \in \Dts$, the optimal action $x^*$ will remain inside the safe decision sets for all rounds after $t$.
 Thus, it is sufficient to find the first round $t$, such that $x^* \in \Dts$. This leads to a shorter duration $T'$ for the pure exploration phase.
In particular, following the arguments in Lemma \ref{lem:xin_app}, it can be shown that $T'$ becomes the smallest value satisfying ${2\sqrt{2}\|B\|L\beta_{T'}}\leq \Delta \sqrt{2\lambda+\lambda_-{T'}}$, which is now a constant independent of $T$. 
\end{remark}
   
\subsubsection{Bound for the case $\Delta=0$} 
     \begin{lemma}[$\rm Term~II$ for $\Delta=0$] \label{lemm:sum of alpha}
Let Assumptions \ref{assum:noise}, \ref{assum:bounded} and \ref{assum:nonempty} hold. Fix any $\delta\in(0,0.5)$ and assume that $T'$ is such that
$
T'\geq t_{\delta}.\nn
$
 Then,  with probability at least $1-\delta$, it holds
  \begin{equation}\label{eq:Delta>0_1}
      \sum_{t=T'+1}^T 1-\alpha_t \leq  \frac{2\sqrt{2}\|B\|L\beta_T (T-T')}{c\sqrt{2\la+\lambda_-T'}}.
  \end{equation}
  Therefore, with probability at least $1-2\delta$, it holds 
   \begin{align}\label{eq:Delta>0_TermII}
 \sum_{t=T'+1}^{T} (\tilde \mu _t ^ \dagger x_t- \mu^ \dagger x^*) \leq \frac{2\sqrt{2}\|B\|L\beta_T (T-T')}{c\sqrt{2\la+\lambda_-T'}}.
 \end{align}
  \end{lemma}
 
   \begin{proof}
Recall from \eqref{eq:TermIIa}, that for any $T'< t \leq T$, with probability at least $1-\delta$ (note that we have conditioned in the event $\Ec$ in \eqref{eq:conditioned}), $\text{Term II}=1-\alpha_t$. Thus, \eqref{eq:Delta>0_TermII} directly follows once we show \eqref{eq:Delta>0_1}. In what follows, we prove \eqref{eq:Delta>0_1}.

The definition of $\alpha_{t}$ in \eqref{eq:alphaa} and the fact that $\mu^\dagger B x^* \leq c$ imply that
  \begin{equation}
  \alpha_t =
     \begin{cases}
      1 &, \text{if}\ \mu^\dagger B x^*+2\beta_t\|Bx^*\|_{A_t^{-1}}\leq c, \\
       \frac{c}{\mu^\dagger B x^*+2\beta_t\|Bx^*\|_{A_t^{-1}}} \geq \frac{c}{c+2\beta_t\|Bx^*\|_{A_t^{-1}}} &, \text{otherwise}. \nonumber
    \end{cases}
\end{equation}
Thus, for all $t\geq T'+1$:
  \begin{equation}
      \alpha_t \geq \frac{c}{c+2\beta_t\|Bx^*\|_{A_t^{-1}}}, \qquad  \nonumber
  \end{equation}
from which it follows,
  \begin{align*}
      1-\alpha_{t} \leq \frac{2\beta_t\|B x^*\|_{A_t^{-1}}}{c+2\beta_t\|B x^*\|_{A_t^{-1}}}
      \leq \frac{2\beta_t}{c}\|B x^*\|_{A_t^{-1}}
      \leq \frac{2\beta_t\|B x^*\|_2}{c\sqrt{{\lamin(A_t)}}} 
      \leq \frac{2\beta_t\|B\|L}{c\sqrt{\lamin(A_{T'+1})}} .
  \end{align*}
  The last two inequalities follow as in \eqref{eq:bb_00}.
To complete the proof, note that since $T'\geq t_{\delta}$, we can apply Lemma \ref{lemm:HT'}. Thus, with probability at least $1-\delta$ it holds,
  \begin{align*}
      \sum_{t=T'+1}^T 1-\alpha_t &\leq \frac{2\beta_T \|B\|L(T-T')}{c\sqrt{\lamin(A_{T'+1})}}\leq \frac{2\sqrt{2}\|B\|L\beta_T (T-T')}{c\sqrt{2\la+\lambda_-T'}},
  \end{align*}
  as desired.
  \end{proof}


\subsection{Completing the proof of Theorem \ref{thm:regretdeltanotzero}}
     
We are now ready to complete the proof of Theorem \ref{thm:regretdeltanotzero}. Let $T$ sufficiently large such that
\begin{align}
T>T'\geq \bigg( \frac{8L^2\|B\|^2\beta_{T}^2}{\lambda_-\,\Delta ^2}-\frac{2\la}{\la_-} \bigg) \,\vee\, t_\delta.\label{eq:T0_app_thm2}
\end{align}
We combine Lemma \ref{lemm:sum of width} (specifically, Eqn. \eqref{eq:TermI_app}), Lemma \ref{lem:xin_app} (specifically, Eqn. \eqref{eq:Delta=0_TermII}), and, the decomposition in \eqref{eq:Terms_App}, to conclude that
\begin{align*}
 R_T = \sum_{t=1}^{T'} r_t+ \sum_{t=T'+1}^T r_t 
&\leq 2T'+2\beta_T\sqrt{2d(T-T') \log \left(\frac{2TL^2}{d(2\la+\lambda_-T')}\right)}.
\end{align*}
Specifically, choosing $T'=\bigg( \frac{8L^2\|B\|^2\beta_{T}^2}{\lambda_-\,\Delta ^2}-\frac{2\la}{\la_-} \bigg) \,\vee\, t_\delta$ in the above, results in 
\begin{align}\label{eq:22}
R_T = \Oc\left(\frac{\|B\|^2}{\la_-\Delta^2}d\sqrt{T} \log T\right),
\end{align}
where the constant in the Big-O notation may only depend on $L,S,R,\lambda$ and $\delta$.

\subsection{Completing the proof of Theorem \ref{thm:worst-case}}

We are now ready to complete the proof of Theorem \ref{thm:worst-case}. Let $T$ sufficiently large such that
\begin{align}
T>T'\geq t_{\delta}.\nn
\end{align}
We combine Lemma \ref{lemm:sum of width} (specifically, Eqn. \eqref{eq:TermI_app}), Lemma \ref{lemm:sum of alpha} (specifically, Eqn. \eqref{eq:Delta>0_TermII}), and, the decomposition in \eqref{eq:Terms_App}, to conclude that
\begin{align*}
 R_T = \sum_{t=1}^{T'} r_t+ \sum_{t=T'+1}^T r_t \leq 
2T'+2\beta_T\sqrt{2d(T-T') \log \left(\frac{2TL^2}{d(2\la+\lambda_-T')}\right)}
 +\frac{2\sqrt{2}\|B\|L\beta_T (T-T')}{c\sqrt{2\la+\lambda_-T'}}.
\end{align*}
%
%
%
%

Specifically, choosing $T'=\left(\frac{\|B\|L \beta_T T}{c \sqrt{2\lambda_-}}\right)^{\frac{2}{3}}\vee t_\delta$ in the above, results in 
\begin{align}
R_T = \Oc\left(\left(\frac{\|B\|}{c}\right)^{\frac{2}{3}} \lambda_-^{-1/3} d ~ T^{2/3}\log T\right),
\end{align}
where as in \eqref{eq:22} the constant in the Big-O notation may only depend on $L,S,R,\lambda$ and $\delta$.


\section{Extension to linear contextual bandits} \label{sec:contextual}
In this section, we present an extension to the setting of $K$-armed contextual bandit. At each round $t\in[T]$, the learner observes a context consisting of $K$ action vectors, $\{y_{t,a}: a\in[K]\} \subset \mathbb{R}^d$ and chooses one action denoted by $a_t$ and observes its associated loss, $\ell_t=\mu^\dagger y_{t,a_t} + \eta_t$. We consider the same constraint \eqref{eq:constraint} which results in a \emph{safe} set of actions at each round $\{y_{t,a}\,|\,a\in [K], \mu^\dagger B y_{t,a} \leq c \}$. The optimal action at round $t$ is denoted by $y_{t,a_t^*}$ where 
\begin{equation}\label{eq:opt_K}
    a_t^* \in \argmin_{a \in [K], \mu^\dagger B y_{t,a}\leq c }\mu^\dagger y_{t,a}.
\end{equation}

If the chosen action at round $t$ is denoted by $x_t:=y_{t,a_t}$ and the optimal one by $x_t^*:=y_{t,a_t^*}$, the cumulative regret over total $T$ rounds will be
\begin{equation}
    R_T = \sum_{t=1}^T \mu^\dagger x_t - \mu^\dagger x_t^*.\nn
\end{equation}


We briefly discuss how \SUCB~extends to the $K$-armed contextual setting with provable regret guarantees under the following assumptions.

First, we need the standard Assumptions \ref{assum:noise} and \ref{assum:bounded} that naturally extend to the linear contextual bandit setting. Beyond these, in order for the safe-bandit problem to be well-defined, we assume that safe actions exist at each round.  Equivalently, the feasible set in \eqref{eq:opt_K} is nonempty and $x_t^*$ is well-defined. 
Moreover, in order to be able to run the pure-exploration phase of \SUCB ~with random actions (that guarantee Lemma \ref{lemm:HT'} holds) we further require that at least one of these safe actions is randomly sampled at each round $t$ (technically, we need this assumption to hold only for rounds $1,\ldots,T'$). These two assumptions are both implied by Assumption \ref{assum:nonemptysafesets} below.
\begin{myassum}[Nonempty safe sets] \label{assum:nonemptysafesets}
Consider the set $\Dw = \{x \in \mathbb{R}^d:\|Bx\|_2\leq\frac{c}{S}\}$. Then, at each round $t$, $N_t\geq 1$ number of  $K$ action vectors lie within $\Dw$. 
\end{myassum}
Finally, in order to guarantee that \SUCB~has sub-linear regret for the $K$-armed linear setting we need that the safety gap at each round is strictly positive.
\begin{myassum}[Nonzero $\Delta$] \label{assum:Deltacontextual}
The safety gap $\Delta_t = c-\mu^\dagger B x_t^*$ at each round $t$ is positive.
\end{myassum}

Under these assumptions, \SUCB~naturally extends to the $K$-armed linear bandit setting. Specifically, at rounds $t\leq T'$, \SUCB~  randomly selects $x_t$ to be one of the available $N_t$ action vectors that belong to the set $\Dw$. Assume that $\lamin(\E[x_tx_t^\dagger])\geq \lambda_->0$ for all $t\in[T']$. 

After round $T'$, \SUCB ~implements the safe exploration-exploitation phase by choosing safe actions based on OFU principle as in \eqref{eq:ofu}. Therefore line \ref{algo:x_t} of \SUCB~  changes to
 \begin{equation}
     a_t = \arg \min_{a\in \Ats} \min_{v \in \Cc_t} v^\dagger y_{t,a},
 \end{equation}
 where the safe set at rounds $t\geq T'+1$ is defined by
\begin{equation} \label{eq:safesetcontextual}
    \Ats = \{a \in [K]: v^\dagger B y_{t,a} \leq c, \forall v\in \Cc_t \}. 
\end{equation}
 
 With these and subject to Assumptions \ref{assum:noise}, \ref{assum:bounded}, \ref{assum:nonemptysafesets} and \ref{assum:Deltacontextual}, it is straightforward to extend the results of Theorem \ref{thm:regretdeltanotzero} to the setting considered here. Namely, under these assumptions, \SUCB~achieves regret $\Otilde(\sqrt{T})$ when $T'$ is set to $T_{\Delta}$ as in \eqref{eq:T0} for $\Delta = \min_{t\in[T]} \Delta_t$. 

\section{\SUCB~with $\ell_1$-confidence region}\label{sec:L1}

In this section we briefly discussed a modified $\ell_1$-confidence region (as in \cite{Dani08stochasticlinear}), which is used in our numerical experiments.

\paragraph{Motivation.}
The minimization in \eqref{eq:ofu} involves solving a bilinear optimization problem. In view of \eqref{eq:C_t} and \eqref{eq:safeset} it is not hard to show that \eqref{eq:ofu} can be equivalently expressed as follows:
\begin{align*}
\tilde \mu _t ^ \dagger x_t\,=\,&\min_{x}~\hat{\mu}_t^\dagger x - \beta_t\,\|x\|_{A_t^{-1}}\quad {\rm sub. to}~~~ \hat{\mu}_t^\dagger B x + \beta_t\,\|Bx\|_{A_t^{-1}} \leq c,~~x\in\Dc_0~.
\end{align*}
This is a non-convex optimization problem. Thus, we present a variant of \SUCB~(and its analysis) and we show that it can be efficiently implemented (particularly so, when the decision set is a polytope)  \cite{Dani08stochasticlinear}. We use this variant in our simulation results (see Appendix \ref{sec:sim}). 

\paragraph{Algorithm and guarantees.} We adapt the procedure first presented in \cite{Dani08stochasticlinear} to our new \SUCB~algorithm. The pure-exploration phase of the algorithm remains unaltered. In the safe exploration-exploitation phase, the only thing that changes is the definition of the confidence region in Line \ref{lineCt} in Algorithm \ref{algo:Safe-LUCB}. Specifically, we 
define the modified $\ell_1$-confidence region as follows:
\begin{align}\label{eq:Cct}
    \Cc_t^{\ell_1} := \{v \in \mathbb R^d:\|v-\hat \mu _t\|_{A_{t,1}}\leq \beta_t\sqrt{d}\}.
    \end{align}
Note that for any $v \in \Cc_t$ and all $t>0$, $\|A_t^{1/2}(v-\hat \mu _t)\|_1 \leq \sqrt{d}\|A_t^{1/2}(v-\hat \mu _t)\|_2\leq\sqrt{d}\beta_t$. Thus,
$
    \Cc_t \subseteq \Cc_t^{\ell_1}, \quad \forall t>0.
$
From this and Theorem \ref{thm:confidence_ball}, we conclude $\Pr(\mu \in   \Cc_t^{\ell_1}, \forall t>0)\geq 1-\delta$. Then, the natural modification of \eqref{eq:ofu} becomes
\begin{align} 
\tilde \mu _t ^ \dagger x_t =& \min_{x\in \Dts , v \in \Cc_t^{\ell_1}} v^\dagger x
=\min_{v \in \Cc_t^{\ell_1}}~~f(v) \label{eq:ofu1},
\end{align}
where
\begin{align}\label{eq:in1}
f(v):=\min_{\substack{x\in\Dc_0\\ \hat{\mu}_t^\dagger B x +\sqrt{d} \beta_t\,\|Bx\|_{A_t^{-1}} \leq C}}~~\nu^\dagger x.
\end{align}

From these, it is clear that all the results and theorems can be directly applied to the modified algorithm which uses $\ell_1$-confidence region in \eqref{eq:Cct}, with $\beta_t\sqrt{d}$ instead of $\beta_t$. As noted in \cite{Dani08stochasticlinear}  the regret of the modified algorithm does not optimally scale with the dimension $d$ (since there is an extra factor of $\sqrt{d}$ introduced by the substitution $\beta_t\leftarrow\beta_t\sqrt{d}$). However, as explained next, solving \eqref{eq:ofu1} is now computationally tractable.

\paragraph{On computational efficiency.}
Note that the minimization in \eqref{eq:in1}
 is a convex program that can be efficiently solved for fixed $\nu$. In particular, if $\Dc_0$ is a polytope then the minimization in \eqref{eq:in1} is a quadratic program.  Moreover, note that $f(v)$ is positive homogeneous of degree one, i.e., $f(\theta v) = \theta f(v)$ for any $\theta\geq 0$. Therefore, in order to solve \eqref{eq:ofu1} it suffices to evaluate the function $f(v)$ at the $2d$ vertices $v_1,\ldots,v_{2d}$ of $\Cc_t^{\ell_1}$ in \eqref{eq:Cct} and choose the minimum $f_{\rm min}:=\min_{v_i,~i\in[2d]} f(v_i)$. In order to see this, let $v^*\in\arg\min_{v\in\Cc_t^{\ell_1}} f(v)$ and $\theta_1,\ldots,\theta_{2d}\geq 0, \sum_{i=1}^d\theta_i=1$ such that $v^*=\sum_{i=1}^{2d}\theta_i v_i$. Then,
 \begin{align}
\min_{v\in\Cc_t^{\ell_1}} f(v) = f(v^*) =  \sum_{i=1}^{2d}\theta_i f(v_i) \geq  f_{\rm min} \sum_{i=1}^{2d}\theta_i = f_{\rm min} \geq \min_{v\in\Cc_t^{\ell_1}} f(v). \nn
 \end{align}
 Thus,
\begin{align}\label{eq:2d}
 \min_{v\in\Cc_t^{\ell_1}} f(v) = \min_{v_i,~i\in[2d]} f(v_i).
\end{align}
 To sum up, we see from \eqref{eq:2d} that solving \eqref{eq:ofu1} amounts to solving $2d$ quadratic programs (when $\Dc_0$ is a polytope).
 

\section{On GSLUCB} \label{algo:GSLUCB}

Having no knowledge of the safety gap $\Delta$,  GSLUCB starts conservatively by setting the length of the pure exploration phase  to its largest possible value, which is equal to $T_0$ defined in Theorem \ref{thm:worst-case} (corresponding to $\Delta = 0$). The idea behind GSLUB is to generate at each round $t$ of the pure-exploration phase a certain value $\Delta_t$ that serves as a lower bound for the unknown safety gap $\Delta$. We discuss possible ways to do so next, but for now let us describe how these  lower estimates of $\Delta$ can be useful. Owing to the result of Theorem \ref{thm:regretdeltanotzero}, at each round $t$, GSLUCB computes a pure exploration duration $T'_t=T_{\Delta_t}$, which is associated with the lower confidence bound $\Delta_t$ (Eqn. \eqref{eq:T0} for $\Delta=\Delta_t$). If at some round $t$, the computed $T'_t$ becomes less than $t$, then Theorem \ref{thm:regretdeltanotzero} guarantees that $x^* \in \Dts$ and the algorithm switches to the exploration-exploitation phase. 


One way to compute the $\Delta_t$'s that guarantees $\Delta_t\leq \Delta$ is as follows.  For each vector $v\in\Cc_t$ denote $x^* _v\in \argmin_{x\in \Ds(v)} v^\dagger x$, where $\Ds(v):=\{x\in \Dc_0: v^\dagger B x \leq c\}$ and define
\begin{equation}\Delta_t := \min_{v \in C_t} \Delta_v, \label{lcbdelta}\end{equation}
where $\Delta_v := c - v^\dagger B x^*_v$. Since $\mu\in\Cc_t$ with high probability (cf. Theorem \ref{thm:confidence_ball}) and by definition of $\Delta$, it can be seen that $\Delta_t\leq \Delta$. Unfortunately, solving \eqref{lcbdelta} can be challenging and, in general, one has to resort to relaxed versions of the optimization involved, but ones that guarantee $\Delta_t\leq \Delta$ (at least after a few rounds). We leave the study of this general case to future work and we discuss here a special case in which this is possible. We have implemented this special case in the simulation results presented in Figure \ref{fig:regret,karmed} (see Appendix \ref{sec:sim}). Specifically, we consider a finite $K$-armed linear bandit setting with feature vectors denoted by $y_1,\ldots,y_K$. 
We produce lower estimates $\Delta_{t}$ as follows. For all $i \in [K]$, we form the following two sets. (i) The set $\Cc_t^{i}=\{v\in \Cc_t~|~v^\dagger B y_i\leq c\}$ of all vectors in the confidence region for which the action $y_i$ is deemed safe; (ii) The set $\mathcal{Y}_t^i=\{y_j, j \in [K]~|~ \max_{v\in \Cc_t^{i}} v^\dagger B y_j \leq c\}$ of all actions that are considered safe with respect to all $v\in\Cc_t^{i}$. Then, we define
\begin{align}
\Delta_t^i := \min_{\substack{v\in\Cc_t^{i}\\ v^\dagger y_i \leq v^\dagger y,\:\text{for all}\: y\in \mathcal{Y}_t^i}}~~c-v^\dagger B y_i.
\end{align}
It can be checked that $\min_{i\in[K]} \Delta_t^i\leq \Delta$. Thus we rely on $ \min_{i\in[K]} \Delta_t^i$ as our lower confidence bound on $\Delta$. Note that computing $ \min_{i\in[K]} \Delta_t^i$ is computationally tractable for finite $K$ and an $\ell_1$ confidence region.

\section{Simulation Results}\label{sec:sim}
In this section, we provide the details of our numerical experiments. 
In view of our discussion in Appendix \ref{sec:L1}, we implement a modified version of \SUCB~  which uses 1-norms instead of 2-norms  (as in \cite{Dani08stochasticlinear}; see also Appendix \ref{sec:L1} for details). We have taken $\delta = 0.01$, $\la=1$, and $R = 0.1$ in all cases. 

\begin{figure}
     \centering
     \begin{subfigure}[b]{0.45\textwidth}
         \centering
         \includegraphics[width=\textwidth]{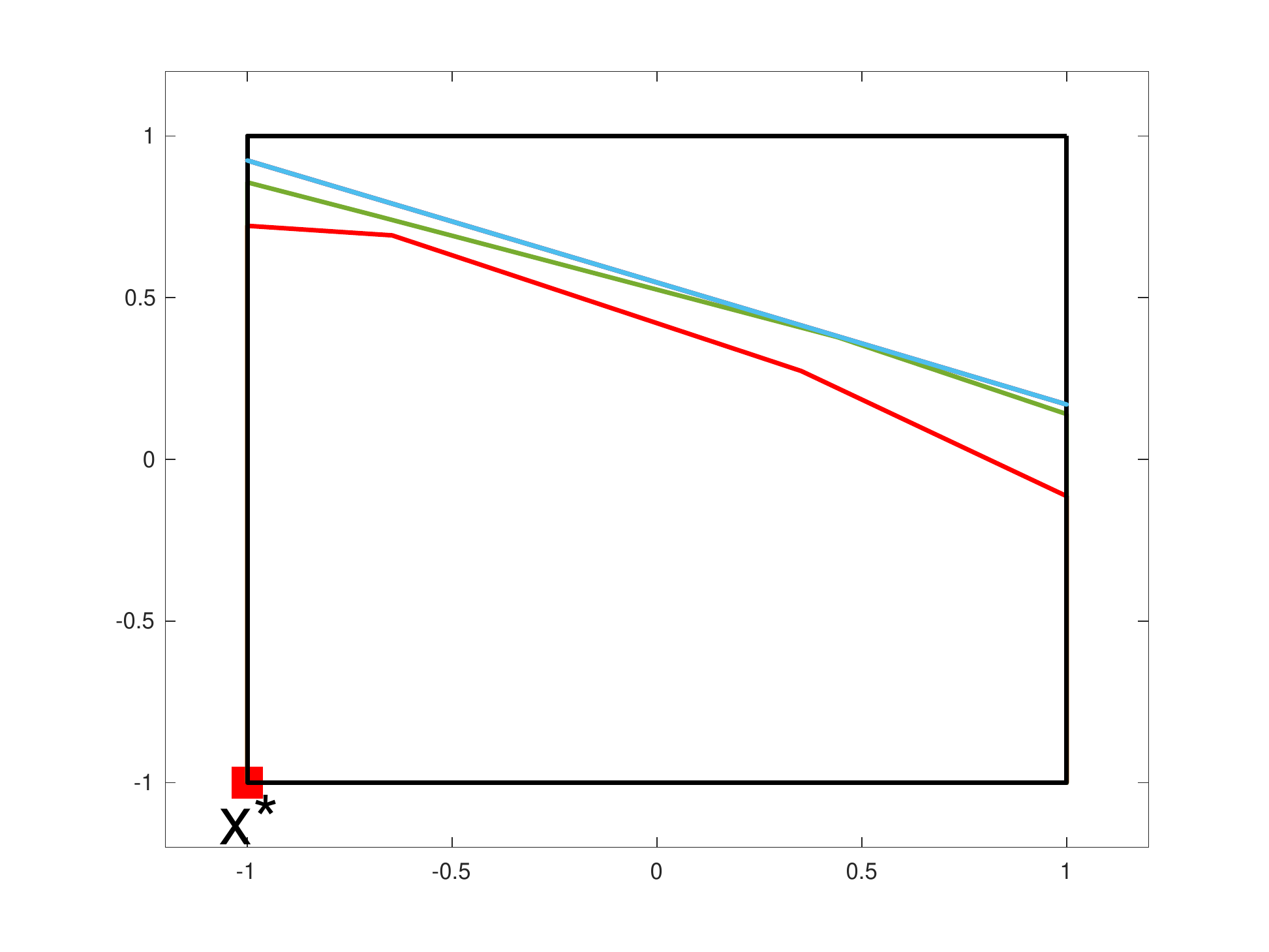}
         \caption{\SUCB~ with pure exploration phase.}
         \label{fig:noexplore}
     \end{subfigure}
    \hfill 
     \begin{subfigure}[b]{0.45\textwidth}
         \centering
         \includegraphics[width=\textwidth]{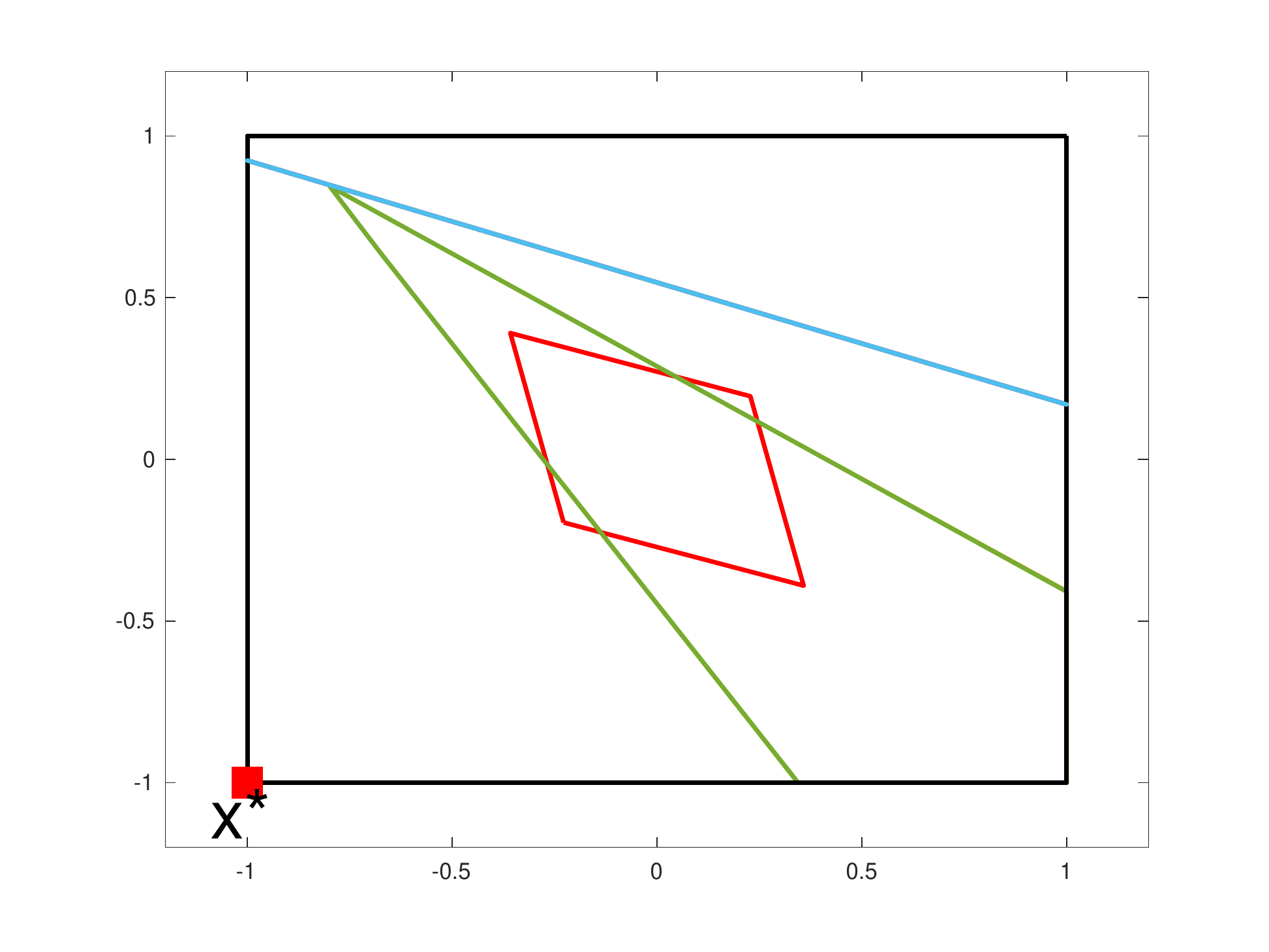}
        \caption{\SUCB ~without pure exploration phase}
         \label{fig:sucb}
     \end{subfigure}
        \caption{Growth of $\Dts$~ with and without pure exploration phase. In both figures: $\Dc_0$ (in black) $\Ds$ (in blue), $\Dc^{S}_{T'+1}$ (in red), $\Dc^{S}_{5e4}$ (in green). Also, shown the optimal action $x^*$. Note that $x^*\in\Dc^{S}_{T'+1}$ when pure exploration phase is used as suggested by Lemma \ref{lem:x^*in}.}
        \label{fig:dsafe}
\end{figure}


\begin{figure} 
     \centering
     \begin{subfigure}[b]{0.3\textwidth}
         \centering
         \includegraphics[width=\textwidth]{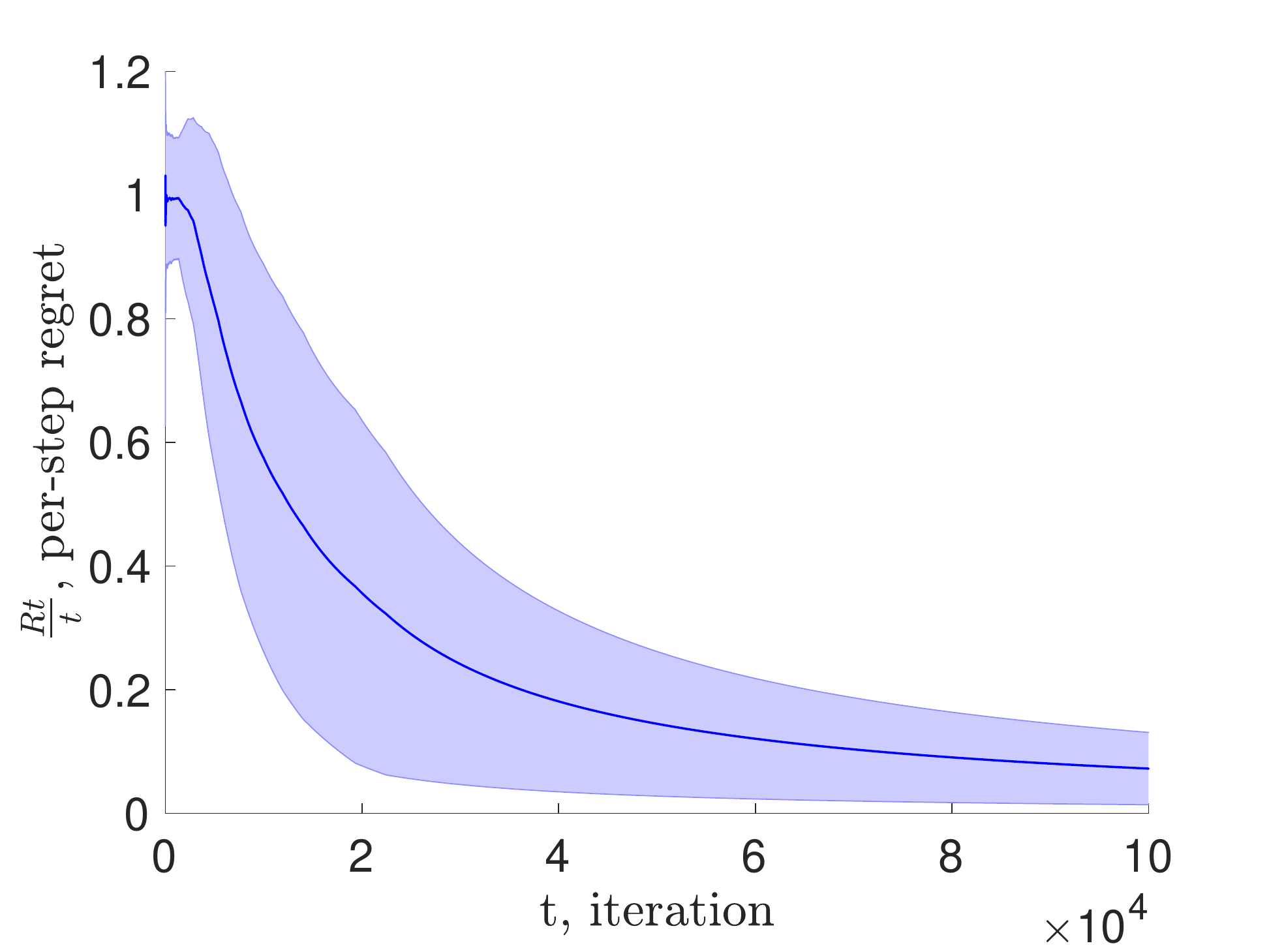}
         \caption{\SUCB,~ $T' = T_{\Delta}$}
         \label{fig:tdeltaerr}
     \end{subfigure}
     \begin{subfigure}[b]{0.3\textwidth}
         \centering
         \includegraphics[width=\textwidth]{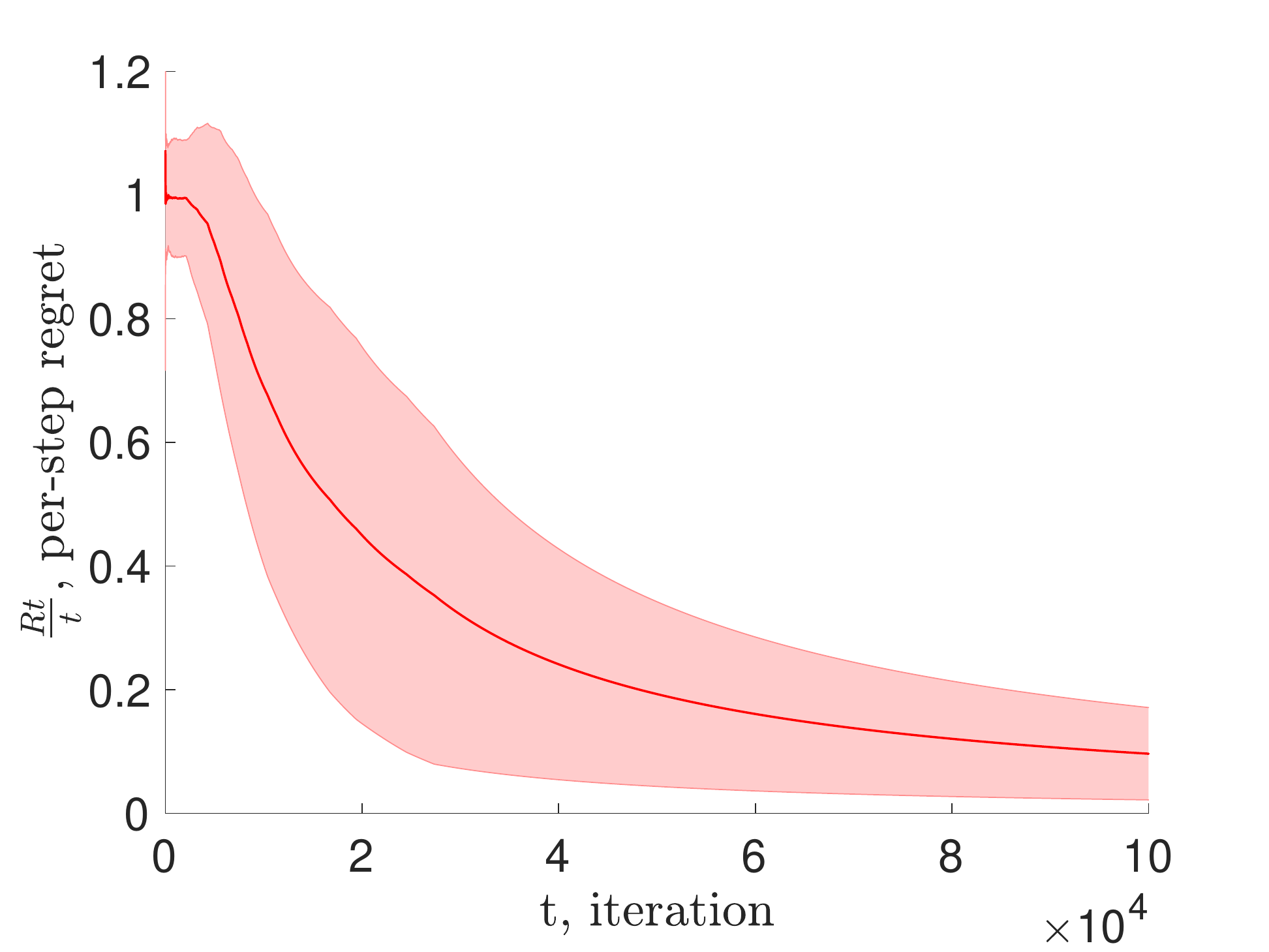}
        \caption{GSLUCB}
         \label{fig:sucb}
     \end{subfigure}
      \begin{subfigure}[b]{0.3\textwidth}
         \centering
         \includegraphics[width=\textwidth]{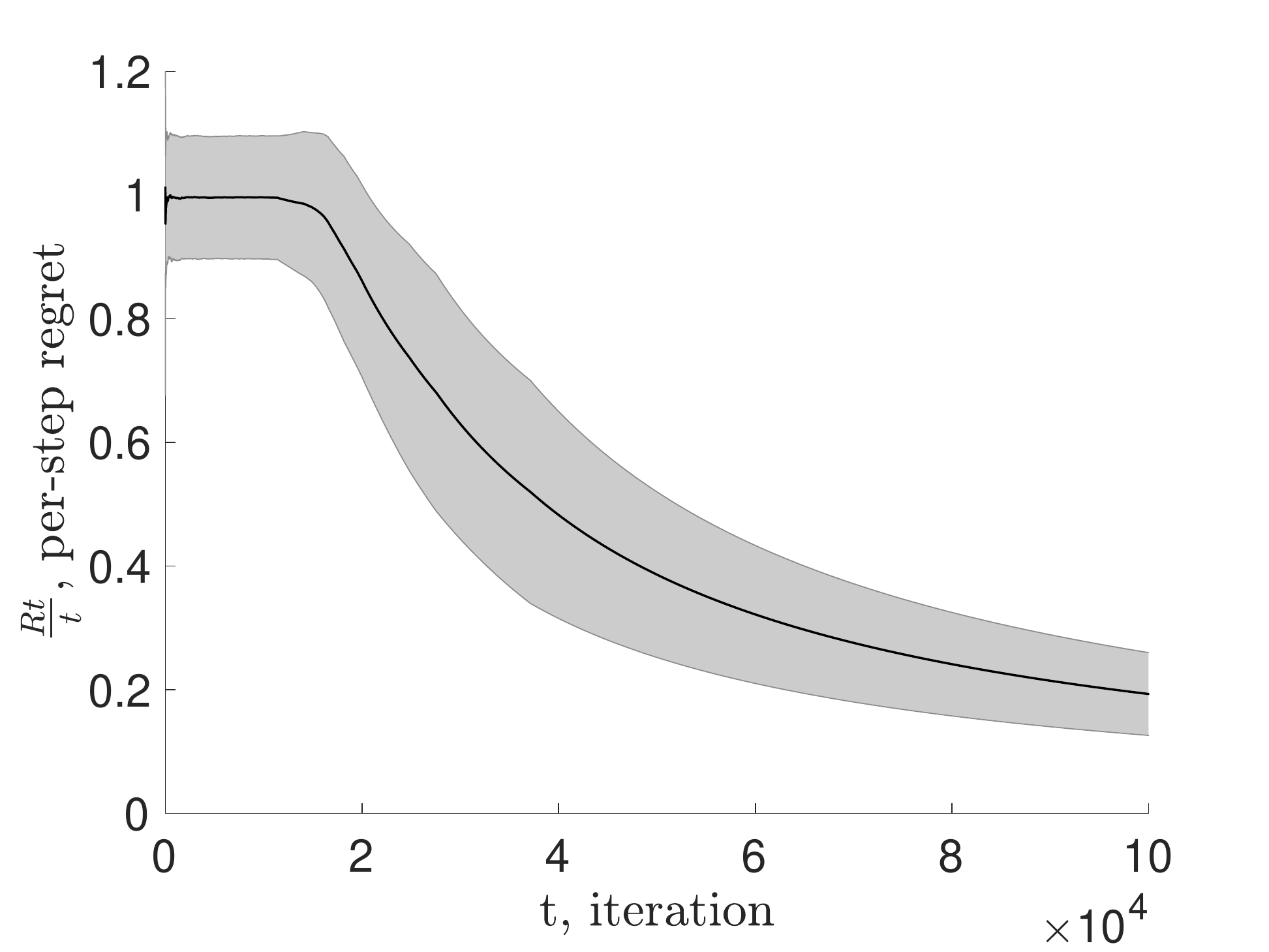}
        \caption{\SUCB, ~$T' = T_0$}
         \label{fig:sucb}
     \end{subfigure}
        \caption{Comparison of mean per-step regret for \SUCB ($T' = T_{\Delta}$), GSLUCB, and \SUCB ($T' = T_0$). The shaded regions show one standard deviation around the mean. The results are averages over 20 problem realizations.}
        \label{fig:errbar}   
\end{figure}

Figure \ref{fig:regret,karmed} compares the average per-step regret of 1) \SUCB  ~with knowledge of $\Delta$; 2) \SUCB  ~without knowledge of $\Delta$ (hence, assuming $\Delta =0$); 3) GSLUCB without knowledge of $\Delta$ (the algorithm creates a lower confidence bound for $\Delta$ as the pure exploration phase runs). Figure \ref{fig:errbar} highlights the sample standard deviation of regret around the average per-step regret for each of the above-mentioned cases.
 We  considered a time independent decision set of 15 arms in $\mathbb{R}^4$  such that 5 of the feature vectors are drawn uniformly from $\Dc^w$ and the other 10 are drawn uniformly from unit ball in $\mathbb{R}^4$. Moreover, $\mu$ is drawn from $\mathcal{N}(0,I_4)$ and then normalized to unit norm. $B$ and $c$ are drawn uniformly from $[0,0.5]^{4\times4}$ and [0,1] respectively. The results shown depict averages over 20 realizations. It can be seen from the figure that GSLUCB performs significantly better than the worst case suggested by Theorem \ref{thm:worst-case} (aka \SUCB~assuming $\Delta =0$). In fact, it appears that it approaches the improved regret performance suggested by Theorem \ref{thm:regretdeltanotzero} of \SUCB~ with knowledge of $\Delta$.  

Our second numerical experiment serves to showcase the value of the safe exploration phase as discussed in Section \ref{sec:Delta>0}. We focus on an instance with positive safety gap $\Delta>0$ to verify the validity of Lemma \ref{lem:x^*in}, namely that $x^*\in\Dts$ for $t\geq T'+1$, when $T'$ is appropriately chosen. Furthermore, we compare the performance with a ``naive" variation of \SUCB~that only implements the safe exploration-exploitation phase (aka, no pure exploration phase). The regret plots of the two algorithms (with and without pure exploration phase) shown in Figure  \ref{fig:regret,polytope} clearly demonstrate the value of the pure exploration phase for the simulated example.  
%
Specifically, for the simulation, we consider a horizon $T = 100000$ with decision set $\Dc_0$ the unit $\ell_{\infty}$-ball in $\mathbb{R}^2$, and, the following parameters: $\mu=\begin{bmatrix}
    0.9 \\0.044 \end{bmatrix}$, $B= \begin{bmatrix}
    0.6& 1.8  \\
    1.8 & 0.4 
  \end{bmatrix}$, $c = 0.9$.  
  We have chosen a low-dimensional instance, because we find it instructive to also depict the the growth of the safe sets for the two algorithms. This is done in Figures \ref{fig:noexplore} and \ref{fig:sucb}, where we illustrate the safe sets of \SUCB~ with and without pure exploration phase, respectively.  Black lines denote the (border of) the polytope $\Dc_0$; blue lines denote the linear constraint in \eqref{eq:constraint}; red lines denote the (border of) $\Dc_{T'+1}^{\text{s}}$, where $T' = T_{\Delta} = 1054$ and $T' = 0$ for  Figures \ref{fig:noexplore} and \ref{fig:sucb}, respectively; and, green lines denote (the border of) safe sets $\Dc_{50000}^{\text{s}}$ at round 50000. Also depicted the optimal action $x^*$ with coordinates $\{-1,-1\}$. As expected, \SUCB ~starts the exploration-exploitation phase with a safe set that  includes $x^*$ while, without the pure exploration phase, the algorithm  starts the exploration-exploitation phase with a smaller safe set which does not include $x^*$ and as a results, fails in expanding the safe set to include $x^*$ even after $T = 50000$ rounds. This results in the bad regret performance in Figure \ref{fig:regret,polytope}.

%
%
%

                  






\end{document}